\documentclass{article}
\usepackage{iclr2027_conference,times}
\iclrpreprintcopy

\usepackage{amsmath,amsfonts,bm}

\def\eqref#1{equation~\ref{#1}}

\def\1{\bm{1}}

\DeclareMathAlphabet{\mathsfit}{\encodingdefault}{\sfdefault}{m}{sl}
\SetMathAlphabet{\mathsfit}{bold}{\encodingdefault}{\sfdefault}{bx}{n}

\usepackage[hidelinks]{hyperref}
\usepackage{url}
\usepackage{graphicx}
\usepackage{amsmath}
\usepackage{amssymb}
\usepackage{amsthm}
\usepackage{booktabs}
\usepackage{multirow}
\usepackage{caption}
\usepackage{float}
\usepackage[linesnumbered,ruled,vlined]{algorithm2e}
\SetKwInput{KwInput}{Input}
\SetKwInput{KwOutput}{Output}

\newtheorem{theorem}{Theorem}[section]
\theoremstyle{definition}

\title{Toward Omni Multimodal Graph Foundation Model: A Topology-Driven Binding Approach}

\author{%
  Xunkai Li$^{1}$ \quad Chenxi Wan$^{1}$ \quad Yinlin Zhu$^{2}$ \quad Wang Luo$^{2}$ \\
  \bfseries Hongchao Qin$^{1}$ \quad Rong-Hua Li$^{1}$\thanks{Correspondence to: Rong-Hua Li \href{mailto:lironghuabit@126.com}{\texttt{<lironghuabit@126.com>}}.} \quad Guoren Wang$^{1}$ \\
  \normalfont $^{1}$Beijing Institute of Technology \\
  \normalfont $^{2}$Sun Yat-sen University%
}

\begin{document}

\maketitle

\begin{abstract}
Multimodal graph foundation models (MGFMs) seek to learn generalizable representations from large-scale graphs with heterogeneous node modalities. However, real-world Multimodal-Attributed Graphs (MAGs) often contain incomplete node attributes, limiting the scale and diversity of available pretraining corpora. Besides, existing MGFMs primarily incorporate graph topology as structural context, overlooking its role in guiding multimodal binding and shaping a unified representation space. To address these challenges, we propose \underline{\textbf{GraphBind}}, a topology-driven approach that uses graph topology to bind rich modality information into a unified shared space. GraphBind is motivated by the stability of graph topology, which provides structural references and complementary semantic information for multimodal binding. Concretely, GraphBind uses topology to organize self semantics and reliable neighborhood semantics into a global shared space that integrates structure and semantics, and adapts this space to discriminative and generative tasks through lightweight interfaces. Extensive experiments against 11 representative baselines demonstrate that GraphBind achieves leading performance on both discriminative and generative tasks, with relative improvements of up to 28.1\% over the strongest baseline.

\end{abstract}

\section{Introduction}

Graph foundation models (GFMs) aim to learn generalizable graph representations through large-scale pretraining across graph domains and downstream tasks~\cite{liu2023towards,wang2025graphfoundationcomprehensive}, showing practical value in cross-domain graph classification and graph instruction tuning~\cite{liu2024oneforall,tang2024graphgpt}. As real-world graphs increasingly contain heterogeneous modalities such as text and images~\cite{zhu2025mosaic,wan2026openmag}, multimodal graph foundation models (MGFMs) have attracted growing attention in recent years. By integrating graph topology with multimodal attributes, MGFMs learn transferable representations for diverse graph tasks while supporting multimodal reasoning and generation~\cite{he2025unigraph2,liu2026planet,sun2026mario}.

Despite substantial progress, learning reliable shared representations from graphs with incomplete node modalities remains challenging, which creates a critical data bottleneck for large-scale MAG pretraining. This bottleneck can be summarized from three perspectives: (1) \textbf{Incomplete Pretraining Corpora.} Real-world MAGs often contain incomplete node attributes. Restricting pretraining to nodes with complete observations substantially limits the scale and diversity of usable graph corpora~\cite{wu2026missingmodality,wang2023shaspec}. (2) \textbf{Fragmented Semantic Spaces.} Separate modality encoding and alignment may produce fragmented semantic spaces, making it difficult to construct the unified representation space required by MGFMs. (3) \textbf{Limited Topological Binding.} Existing MGFMs primarily incorporate graph topology as structural context, with limited exploration of its role in guiding multimodal binding and shaping a unified representation space~\cite{he2025unigraph2,liu2026planet,sun2026mario}. Our empirical study in Sec.~\ref{sec:empirical_study} shows that graph neighborhoods provide additional semantic signals beyond node-local modalities. This finding supports the use of topology as a stable reference when learning from pretraining corpora with incomplete modalities.

Building upon this observation, we propose \textbf{GraphBind}, a novel MGFM framework that binds heterogeneous modality semantics through graph topology into a unified representation space. The key insight of GraphBind is that graph topology provides a stable structural reference and useful neighborhood information for multimodal binding. Concretely, GraphBind consists of three coupled designs. (1) \textbf{Incomplete Corpus Pretraining.} GraphBind explicitly introduces modality absence into pretraining, enabling the model to fully utilize incomplete MAGs as pretraining corpora and improve its robustness. (2) \textbf{Topology-Driven Multimodal Binding.} GraphBind uses graph topology as structural guidance to organize node semantics, neighborhood context, and multimodal attributes into a unified representation space. (3) \textbf{Unified Adaptation.} GraphBind adapts the unified representation space to both discriminative and generative tasks through lightweight fine-tuning while preserving the shared semantic interface learned during pretraining.

\textbf{Our Contributions:} (1) \textbf{In-depth Investigation.} We revisit realistic modality absence in MAGs and empirically show that graph neighborhoods provide useful recovery signals beyond the available attributes of each node. (2) \textbf{Novel Framework.} We propose GraphBind, a topology-guided multimodal graph foundation model that integrates missing-modality pretraining, topology-driven semantic recovery, and unified representation adaptation. (3) \textbf{SOTA Performance.} Extensive experiments against eleven representative baselines demonstrate leading performance on discriminative and generative tasks, with relative improvements of up to 5.4\% in MRR and 28.1\% in BLEU-4, respectively.

\section{Preliminaries}
\label{sec:preliminaries}

\subsection{Problem Formulation}

A multimodal attributed graph (MAG) is denoted as $\mathcal{G}=(\mathcal{V},\mathcal{E},\{\mathbf{X}^{(m)}\}_{m\in\mathcal{M}},\mathbf{A})$, where $\mathcal{V}=\{v_i\}_{i=1}^{N}$ is the node set, $\mathcal{E}$ is the edge set, $N=|\mathcal{V}|$, and $\mathbf{A}\in\mathbb{R}^{N\times N}$ is the adjacency matrix. Each node $v_i$ can have multiple modalities from the modality set $\mathcal{M}$, such as text and images. For modality $m\in\mathcal{M}$, $\mathbf{x}_i^{(m)}\in\mathbb{R}^{d_m}$ denotes the raw attribute of node $v_i$, and $\mathbf{X}^{(m)}\in\mathbb{R}^{N\times d_m}$ denotes the modality feature matrix. 

Given a pretraining collection of MAGs $\mathbb{G}_{\mathrm{pre}}$, an MGFM $f_{\Theta}$ learns generalizable graph representations through self-supervised pretraining across graphs. The pretrained model maps each node $v_i$ to a unified representation $\mathbf{z}_i\in\mathbb{R}^{d}$ and is adapted with task-specific modules. We consider two task families. Discriminative tasks include node classification and link prediction. Generative tasks include Graph to Text Generation (G2Text) and Graph to Image Generation (G2Image). Detailed task descriptions are provided in Appendix~\ref{app:experiment_setup}.

\subsection{Graph Foundation Models}

Graph foundation models learn transferable representations through pretraining across graph domains and tasks~\cite{liu2023towards,wang2025graphfoundationcomprehensive}. Existing methods represent graph computation with task trees or transferable vocabularies~\cite{wang2025git,wang2024gft}, learn reusable graph tokens~\cite{xia2024opengraph}, or model graphs in geometric spaces~\cite{sun2025riemanngfm}. These methods support broad transfer across graph tasks, while their pretraining mainly focuses on graph structure and a single attribute space. MAGs additionally require a shared space that connects heterogeneous node attributes with topology.

\subsection{Multimodal Graph Learning}

Multimodal graph learning combines graph topology with node attributes such as text and images~\cite{zhu2025mosaic,wan2026openmag}. Earlier methods use graph convolution, attention, or global graph structures to fuse modalities for specific tasks~\cite{mmgcn,mgat,lgmrec}, while recent methods jointly update topology and modality representations or assign different roles to neighbors~\cite{zhu2026tmte,zuo2026rolemag}. Multimodal graph foundation models further learn shared representations for transfer across datasets and tasks~\cite{he2025unigraph2,liu2026planet}. Together, these studies explore how node attributes and graph relations can be jointly modeled to support prediction within individual graphs and transfer across graph domains and tasks.

\section{Empirical Investigation}
\label{sec:empirical_study}

In this section, we empirically examine why graph topology can serve as a structural reference for multimodal graph foundation modeling. We focus on two research questions: \textbf{Q1}: Are node-local modalities sufficient for multimodal binding? \textbf{Q2}: Can graph topology support modality recovery when neither modality of the center node is available?

\vspace{+0.1cm}
\noindent \textbf{Datasets, Methods, and Tasks.} We conduct two recovery studies to analyze the semantic roles of node-local modalities and graph topology. We compare GraphBind with MMGCN, UniGraph2, and MGAT on modality feature recovery. Recovery quality is measured by the cosine similarity between the predicted feature and the corresponding reference feature. The first study is conducted on Grocery under single-modality absence, where either the textual or visual modality of the center node is unavailable. We compare two recovery settings with identical target-modality masks. \textit{Single-Modality} (\textbf{SM}) recovery uses only the center node's remaining modality, with neighborhood messages and topology inputs disabled. \textit{Graph-Assisted} (\textbf{GA}) recovery uses available features from one-hop neighbors. Neither setting receives the center node's target modality. The second study is conducted on RedditS when both modalities are absent, where both textual and visual modalities of the center node are unavailable. We use a feature-mean prototype as the non-structural reference and evaluate GA recovery for textual and visual semantics from surrounding nodes.

\begin{figure*}[t]
    \centering
    \begin{minipage}{0.48\textwidth}
        \centering
        \includegraphics[width=\linewidth]{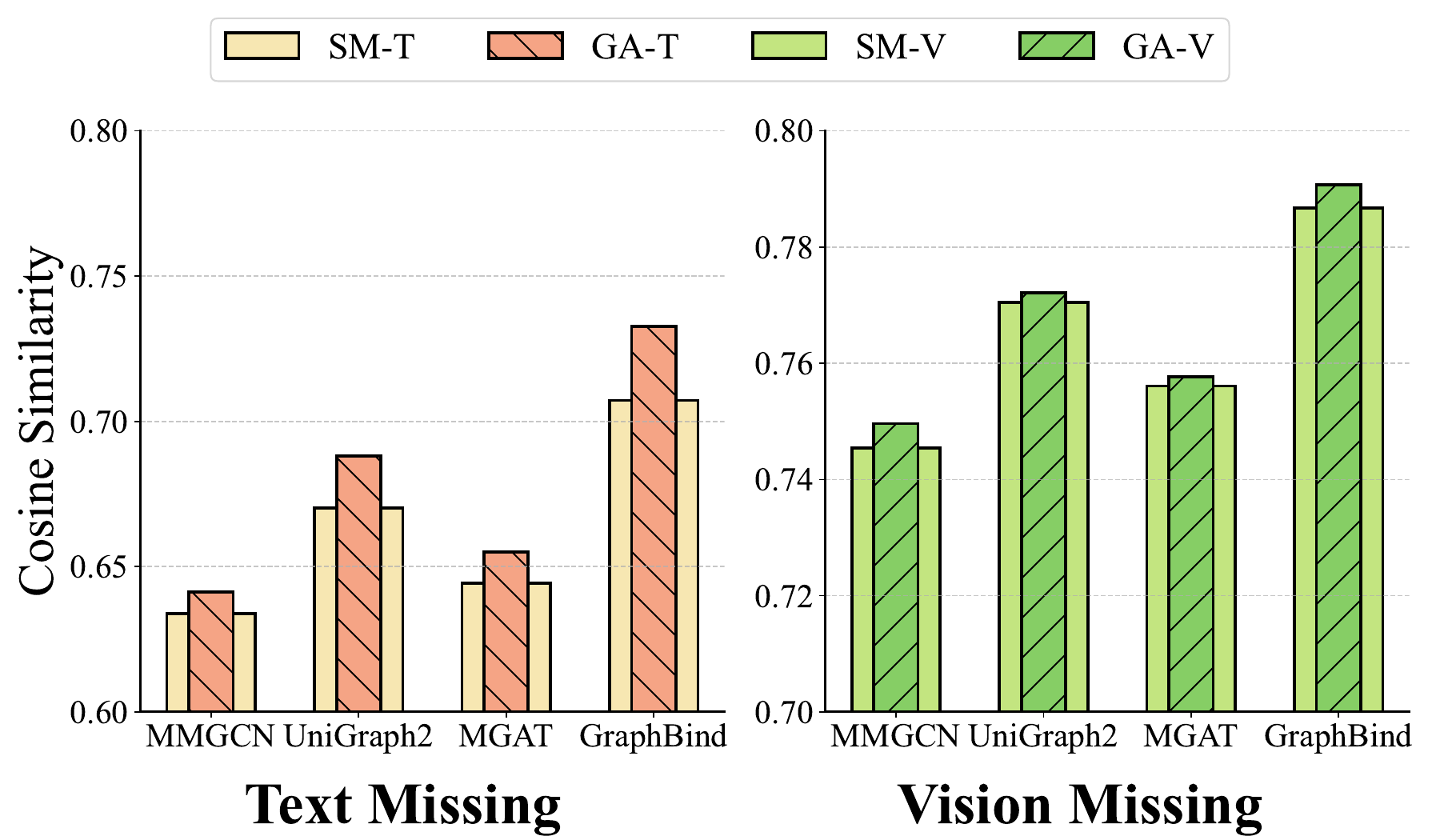}
        \caption{Comparison of recovery settings under single-modality absence on Grocery. \textit{T} and \textit{V} denotes text and vision recovery, respectively.}
        \label{fig:emp1}
    \end{minipage}
    \hfill
    \begin{minipage}{0.48\textwidth}
        \centering
        \includegraphics[width=\linewidth]{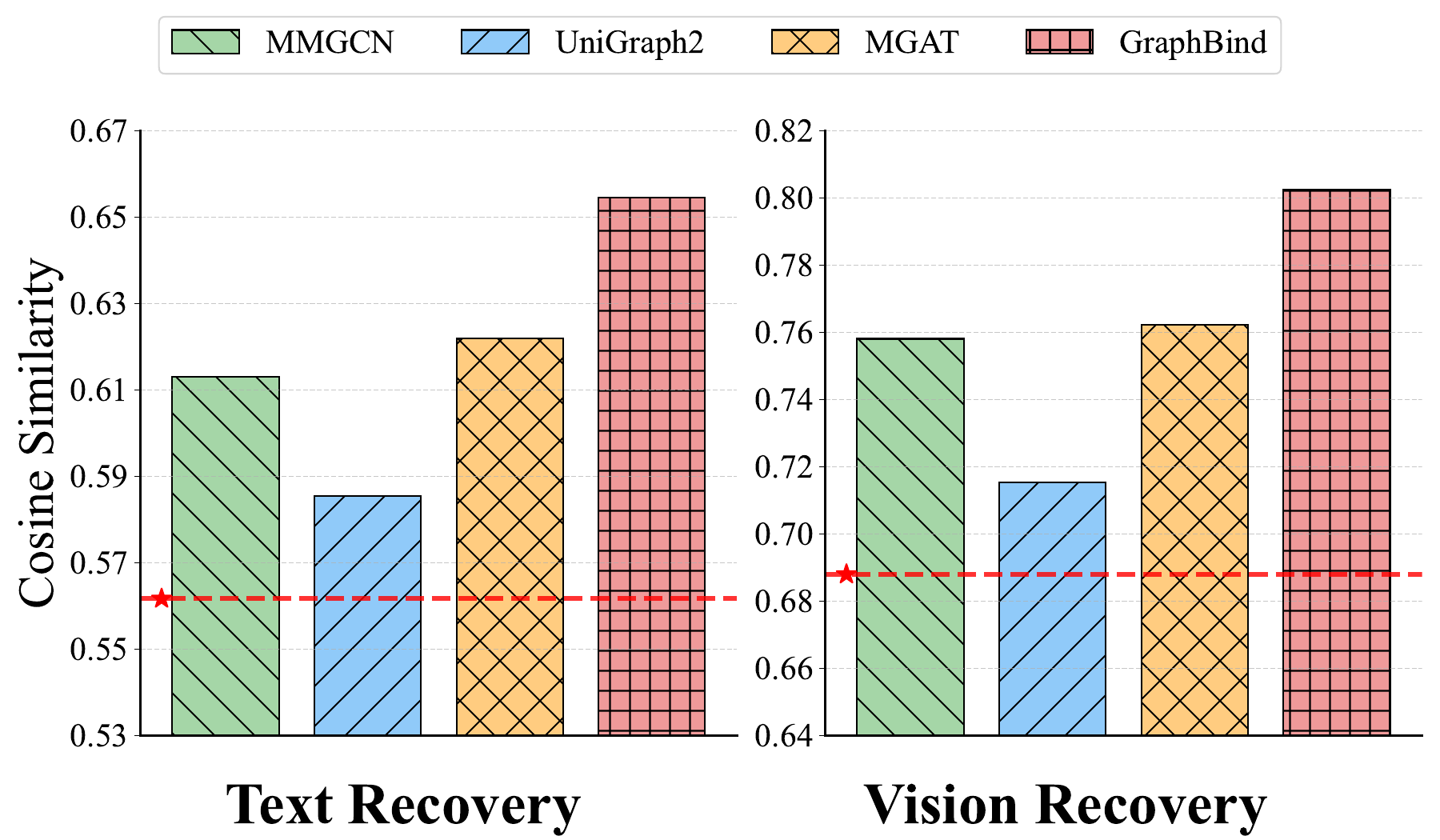}
        \caption{Graph-Assisted semantic recovery under both-modality absence on RedditS. Red lines denote the feature-mean prototype.}
        \label{fig:emp2}
    \end{minipage}
\end{figure*}

\vspace{+0.1cm}
\noindent \textbf{Individual node modalities provide limited semantic information (Answer for Q1).} We first examine whether the remaining modality of a node provides sufficient information for semantic recovery. As shown in Fig.~\ref{fig:emp1}, SM consistently achieves lower recovery similarity than GA across all four methods under both absence settings. This observation indicates that the remaining modality provides useful but limited binding information. The limitation is especially clear when textual semantics need to be inferred from visual attributes, where the available visual information only partially reflects the missing textual semantics. Therefore, combining available node attributes with neighborhood information supports more effective multimodal binding in MAGs.

\vspace{+0.1cm}
\noindent \textbf{Graph topology provides shared semantic information (Answer for Q2).} We further examine whether topology can provide semantic information when the center node has no available modality. As shown in Fig.~\ref{fig:emp2}, all graph-assisted methods achieve higher recovery similarity than the feature-mean prototype for both text and vision when neither modality of the center node is available. This result shows that graph neighborhoods preserve semantic information that is shared across nodes and transferable across target modalities. This observation suggests that graph topology can act as a structural reference for organizing surrounding node semantics in MAGs.

\vspace{+0.1cm}
The two observations above motivate a topology-driven approach to multimodal graph foundation modeling. The available modalities provide useful but limited information for recovering missing modalities, as different modalities capture different aspects of node semantics. In comparison, graph topology provides a structural context for integrating multimodal information across connected nodes. These findings suggest that graph topology can serve as a structural reference for integrating multimodal information across nodes into shared semantic representations. This supports MGFM pretraining on graph corpora with incomplete node modalities.

\section{Methodology}
\label{sec:method}

\begin{figure*}[!t]
    \centering
    \includegraphics[width=1.0\linewidth]{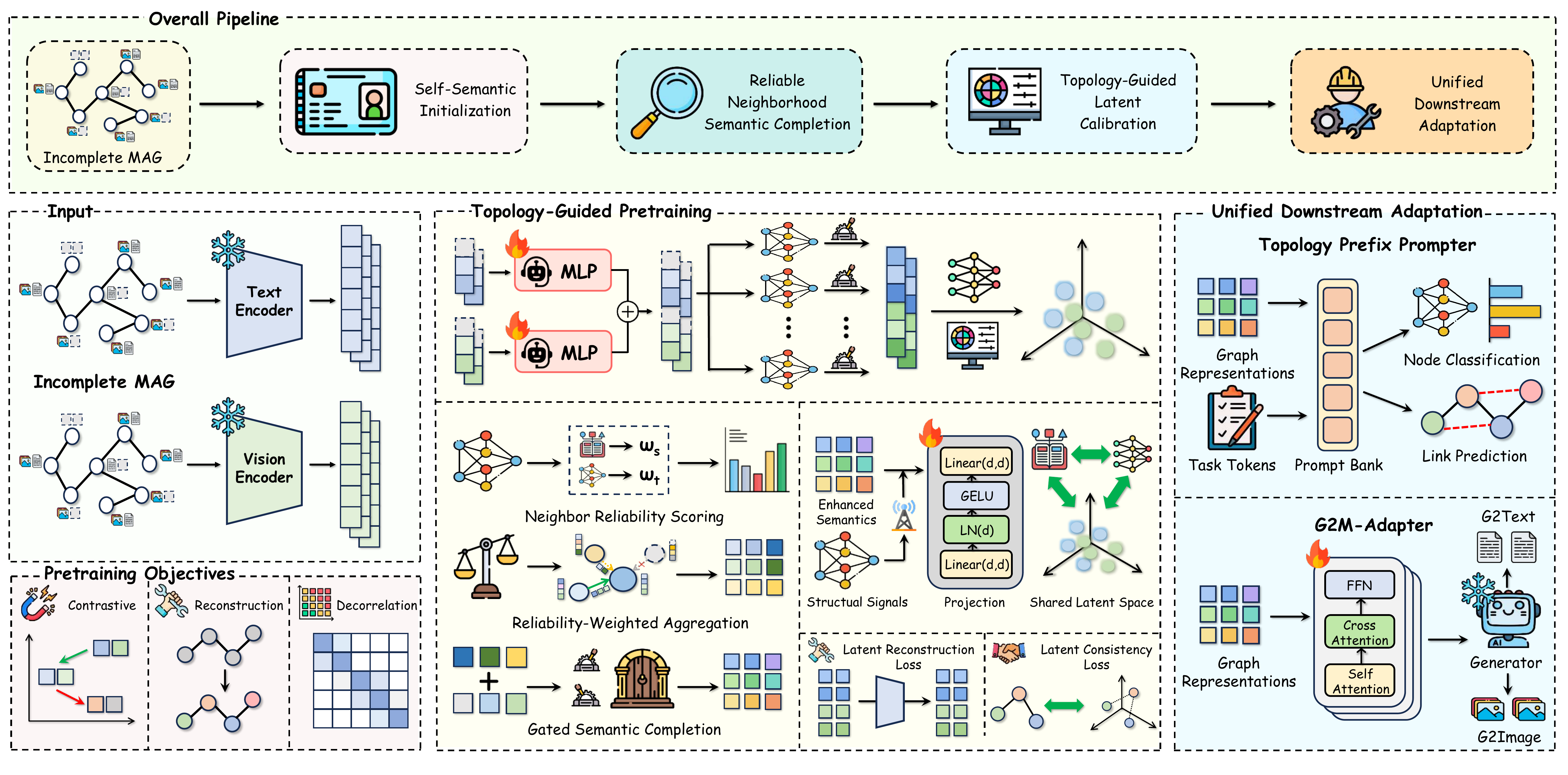}
    \caption{The overall framework of GraphBind.}
    \label{fig:framework}
\end{figure*}

In this section, we present \textbf{GraphBind}, a multimodal graph foundation model that uses graph topology as a structural reference to bind node modalities into a unified representation space. As illustrated in Fig.~\ref{fig:framework}, we first simulate realistic modality absence during pretraining, then construct a topology-guided shared representation through self-semantic initialization, reliable neighborhood aggregation, and topology-guided latent calibration, and finally adapt the learned shared space to discriminative and generative tasks with lightweight task interfaces.

\vspace{+0.1cm}
\subsection{Modality Absence Simulation}

Large-scale MAG pretraining cannot rely only on nodes with complete modalities, since naturally collected graph corpora have uneven modality availability. GraphBind therefore controls modality availability for each node, allowing incomplete MAGs to participate in pretraining. 

Inspired by~\cite{wu2026missingmodality}, we simulate realistic modality absence during pretraining. Specifically, we consider three simulation protocols: \textit{Independent Absence} randomly removes each modality. \textit{Neighborhood Correlated Absence} removes modalities from adjacent nodes to represent local acquisition failure. \textit{Property Dependent Absence} sets the absence probability according to node properties such as degree or community membership. The input availability received by the encoder is defined as
\begin{equation}
\tilde{a}_i^{(m)}
=
a_i^{(m)}(1-s_i^{(m)})(1-r_i^{(m)}),
\label{eq:effective_availability}
\end{equation}
where $a_i^{(m)}$, $s_i^{(m)}$, and $r_i^{(m)}$ indicate original availability, simulated absence, and reconstruction masking. The encoder receives only observed and unmasked modalities, while the original features remain supervision targets. Sampling details are provided in the Appendix~\ref{app:detailed_theoretical_analysis}.

\vspace{+0.1cm}
\subsection{Topology-Guided Pretraining}

\vspace{+0.1cm}
\noindent \textbf{Incomplete Self-Semantic Initialization.}
Under modality absence, the remaining attributes of a node form an incomplete semantic state. We extract modality features with frozen encoders, project them into a common hidden space, and average only the modalities visible to the encoder:
\begin{equation}
\hat{\mathbf{x}}_i^{(m)}=P_m(\mathbf{x}_i^{(m)}),\quad
\mathbf{h}_i^{(0)}=
\frac{\sum_{m\in\mathcal{M}}\tilde{a}_i^{(m)}
\hat{\mathbf{x}}_i^{(m)}}
{\max(1,\sum_{m\in\mathcal{M}}\tilde{a}_i^{(m)})},
\label{eq:self_initialization}
\end{equation}
where $P_m$ is the modality projector and $\mathbf{h}_i^{(0)}$ is the initial self semantic state. For intuition, Eq.~\ref{eq:self_initialization} initializes node $v_i$ from the semantic information available under the current absence pattern. The averaging operation removes scale differences caused by varying numbers of observed modalities.

\vspace{+0.1cm}
\noindent \textbf{Reliable Neighborhood Semantic Completion.}
Node-local observations alone may provide insufficient information for multimodal binding. Although graph neighborhoods can supply missing semantics, not every neighbor provides relevant information. We therefore use topology to score and select reliable neighbors before aggregation. Let $\mathbf{t}_i=f_{\mathrm{topo}}(\mathbf{p}_i)$ be the topology encoding of node $v_i$. For an edge $(j\rightarrow i)$, GraphBind scores neighbor reliability from self semantics, neighbor semantics, topology encodings, their interactions, and local edge strength:
\begin{equation}
r_{ij}
=
\sigma(\mathbf{w}_r^{\top}\boldsymbol{\phi}_{ij}),
\quad
\boldsymbol{\phi}_{ij}
=
\Phi(\mathbf{h}_i^{(0)},\mathbf{h}_j^{(0)},\mathbf{t}_i,\mathbf{t}_j,\omega_{ij}),
\label{eq:reliability_score}
\end{equation}
where $\omega_{ij}$ is a normalized diffusion weight, $\Phi(\cdot)$ concatenates and projects the input features through a GELU layer, and $r_{ij}$ measures whether neighbor $v_j$ is useful for node $v_i$. The reliable message and completed state are computed as
\begin{equation}
\begin{aligned}
\alpha_{ij}
&=
\frac{r_{ij}}{\sum_{k\in\mathcal{N}(i)}r_{ik}+\epsilon},
\quad
\mathbf{m}_i
=
\sum_{j\in\mathcal{N}(i)}\alpha_{ij}\mathbf{W}_n\mathbf{h}_j^{(0)},\\
\mathbf{s}_i
&=
\mathbf{g}_i\odot\mathbf{W}_s\mathbf{h}_i^{(0)}
+
(1-\mathbf{g}_i)\odot\mathbf{m}_i,
\end{aligned}
\label{eq:semantic_completion}
\end{equation}
where $\mathbf{g}_i$ is a learned feature-wise gate balancing self and neighborhood semantics, and $\mathbf{s}_i$ is the completed semantic state. The reliability weights prioritize neighbors that provide relevant semantic and structural information.

\vspace{+0.1cm}
\noindent \textbf{Topology-Guided Shared Latent Calibration.}
The completed state combines local node and neighborhood semantics, but it lacks global structural information. Therefore, we project it into a shared latent space and use topology to generate feature-wise calibration parameters:
\begin{equation}
\mathbf{u}_i=P_{\mathrm{sh}}(\mathbf{s}_i),\quad
[\boldsymbol{\gamma}_i,\boldsymbol{\beta}_i]=C_{\mathrm{topo}}(\mathbf{t}_i),\quad
\mathbf{z}_i^{\mathrm{anc}}=\boldsymbol{\gamma}_i\odot\mathbf{u}_i+\boldsymbol{\beta}_i,
\label{eq:topology_calibration}
\end{equation}
where $P_{\mathrm{sh}}$ projects the completed state into the shared space and $C_{\mathrm{topo}}$ produces calibration parameters. This calibration places node semantics in a common coordinate system. To prevent this projection from discarding useful local semantics, we add two lightweight constraints over reliable neighbors $R(i)$ selected by their reliability scores:
\begin{equation}
\begin{aligned}
\mathcal{L}_{\mathrm{lrec}}
&=
\frac{1}{N}\sum_i
\|D_{\mathrm{sh}}(\mathbf{u}_i)-\operatorname{sg}(\mathbf{s}_i)\|_2^2,\\
\mathcal{L}_{\mathrm{lcons}}
&=
\sum_i\sum_{j\in R(i)}
r_{ij}\|\Delta\mathbf{z}_{ij}^{\mathrm{anc}}-\Delta\mathbf{s}_{ij}\|_2^2 ,
\end{aligned}
\label{eq:latent_constraints}
\end{equation}
where $\operatorname{sg}(\cdot)$ denotes stop-gradient, $\Delta\mathbf{z}_{ij}^{\mathrm{anc}}=\mathbf{z}_i^{\mathrm{anc}}-\mathbf{z}_j^{\mathrm{anc}}$, and $\Delta\mathbf{s}_{ij}=\mathbf{s}_i-\mathbf{s}_j$. The reconstruction constraint retains completed semantics after projection, and the consistency constraint keeps local geometry stable after calibration. Together, they ensure that topology organizes multimodal semantics without collapsing informative local variation. Due to space limitations, the complete pretraining objective and individual loss terms are detailed in Appendix~\ref{app:detailed_theoretical_analysis}.

\vspace{+0.1cm}
\subsection{Unified Downstream Adaptation}

\vspace{+0.1cm}
Based on the shared space constructed during pretraining, we aim to adapt GraphBind to different downstream tasks through lightweight tuning without disrupting the pretrained representation geometry. For discriminative tasks, we design a topology prefix prompter to perform task-specific geometric adjustment. For generative tasks, the G2M-Adapter converts shared graph representations into continuous token embeddings for language and image generators.

\vspace{+0.1cm}
\noindent \textbf{Topology Prefix Prompter.}
For discriminative tasks, different objectives such as node classification and link prediction may require different local decision geometries. Inspired by continuous prompt learning~\cite{zhou2022coop} and graph prompt tuning~\cite{liu2023graphprompt,sun2023allinone}, we introduce a shared prompt bank $\mathbf{P}=[\mathbf{p}_1,\ldots,\mathbf{p}_K]^{\top}$ and a task token $\mathbf{c}_s$. The prompt bank provides reusable task context, while $\mathbf{c}_s$ selects the context needed by a specific task:
\begin{equation}
b_{ik}^{(s)}=\operatorname{Attn}([\mathbf{z}_i^{\mathrm{anc}};\mathbf{c}_s],\mathbf{p}_k),\quad
\bar{\mathbf{p}}_i^{(s)}=\sum_{k=1}^{K}\operatorname{Softmax}_k(b_{ik}^{(s)})\mathbf{p}_k,
\label{eq:prompt_attention}
\end{equation}
The prompted representation is obtained by adding the selected prompt context and task token to $\mathbf{z}_i^{\mathrm{anc}}$, followed by a small residual update. This adapter changes the representation only through a lightweight task-specific path, allowing the downstream head to specialize while the pretrained shared space remains stable.

\vspace{+0.1cm}
\noindent \textbf{G2M-Adapter.}
Generative tasks require graph side information to be expressed as context tokens compatible with pretrained language or visual generators. This interface should retain the node semantics and neighborhood information learned during pretraining. Inspired by lightweight modality bridging~\cite{li2023blip2}, we collect the target node and its one-hop neighbors into a graph context sequence $\mathbf{Z}_{\mathrm{topo}}$. Gated multihead attention converts their shared representations into tokens for the generator:
\begin{equation}
\mathbf{E}_s=F_s(\mathbf{Z}_{\mathrm{topo}}),\quad \bar{\mathbf{Z}}_s=O_s\!\left(\mathbf{Q}_s+\alpha_s\operatorname{MHA}(\mathbf{Q}_s,\mathbf{E}_s,\mathbf{E}_s)\right),
\label{eq:topology_condition_sequence}
\end{equation}
where $s$ denotes text or image generation, and $F_s$ and $O_s$ denote the graph context encoder and output token mapping, respectively. $\mathbf{Q}_s$ contains learnable queries for text and equals $\mathbf{E}_s$ for images; $\alpha_s$ is a scalar gate initialized to zero. For intuition, attention allows each query to aggregate relevant node information, while the gate controls the contribution of the aggregated context. The resulting tokens serve as language input prefixes or diffusion model inputs, allowing the generator to use both target-node semantics and neighborhood information. This design reduces the representation gap between continuous graph features and the token interface required by the generator.

\vspace{+0.1cm}
\subsection{Theoretical Analysis}

\subsubsection{How Do Reliable Neighbors Improve Robustness?}

\begin{theorem}[Reliable Neighbor Robustness]
\label{thm:reliable_neighbor}
Let $R(i)$ and $U(i)=\mathcal{N}(i)\setminus R(i)$ denote the reliable and less reliable neighbors of node $i$. Suppose their scores satisfy $r_{ij}\geq r_+$ on $R(i)$ and $r_{ij}\leq r_-$ on $U(i)$, where $r_+>r_-$. Consider two completion inputs with the same center-node state and reliability scores that differ only in the transformed representations of neighbors in $U(i)$, each by at most $\delta$. Their completed states satisfy
\begin{equation}
\|\mathbf{s}_i-\mathbf{s}_i'\|_2
\leq
\frac{|U(i)|r_-}{|R(i)|r_++|U(i)|r_-+\epsilon}\,\delta.
\label{eq:reliable_neighbor_bound}
\end{equation}
\end{theorem}
Theorem~\ref{thm:reliable_neighbor} formalizes the filtering role of reliability scoring. Reliability-weighted aggregation prevents less reliable neighbors from dominating the completion process, allowing topology to supply complementary semantics without making the completed state overly sensitive to irrelevant or noisy neighborhood information. This result explains why GraphBind estimates neighbor reliability in Eq.~\ref{eq:reliability_score} before performing semantic completion in Eq.~\ref{eq:semantic_completion}.

\subsubsection{How Is Shared Geometry Preserved?}

\begin{theorem}[Shared Geometry Preservation]
\label{thm:shared_geometry}
Let $W_R=\sum_i\sum_{j\in R(i)}r_{ij}>0$. When $i$ is sampled uniformly from $\mathcal{V}$, the latent constraints ensure for any $\kappa\in(0,1)$ that
\begin{equation}
\begin{aligned}
\Pr_i\!\left[
\|D_{\mathrm{sh}}(\mathbf{u}_i)-\mathbf{s}_i\|_2
>\sqrt{\mathcal{L}_{\mathrm{lrec}}/\kappa}
\right]
&\leq\kappa,\\
\left(
\frac{1}{W_R}\sum_i\sum_{j\in R(i)}r_{ij}
\left|\|\Delta\mathbf{z}_{ij}^{\mathrm{anc}}\|_2-\|\Delta\mathbf{s}_{ij}\|_2\right|^2
\right)^{1/2}
&\leq\sqrt{\mathcal{L}_{\mathrm{lcons}}/W_R}.
\end{aligned}
\label{eq:shared_geometry_bound}
\end{equation}
\end{theorem}
The first result converts $\mathcal{L}_{\mathrm{lrec}}$ into a node-level guarantee. A small reconstruction loss limits the fraction of nodes whose projected states cannot recover their completed semantics. The second result bounds the weighted average change in reliable edge lengths after calibration through $\mathcal{L}_{\mathrm{lcons}}$. Complete proofs of both theorems are provided in Appendix~\ref{app:theorem_proof}.

\section{Experiments}
\label{sec:experiments}

In this section, we first describe the experimental setup. Then, we conduct comprehensive evaluations to answer the following research questions: \textbf{Q1}: How does GraphBind perform against representative methods on discriminative and generative tasks? \textbf{Q2}: What are the contributions of the main components in GraphBind? \textbf{Q3}: How robust is GraphBind under varying modality absence patterns? \textbf{Q4}: What are the computational and storage costs of GraphBind?

\subsection{Experimental Setup}
\label{sec:experimental_setup}

\noindent \textbf{Datasets.} We evaluate GraphBind on eight representative MAG datasets from the OpenMAG benchmark~\cite{wan2026openmag} covering e-commerce, social media, video, image-text, and art domains. Specifically, Grocery, Movies, and Toys~\cite{Amazon2018}, together with RedditS~\cite{RedditS}, are used for node classification; DY and Bili\_Dance~\cite{ninerec} for link prediction; Flickr30k~\cite{flickr30k} for graph-to-text generation; and SemArt~\cite{semart} for graph-to-image generation. Detailed dataset statistics are provided in the Appendix~\ref{app:dataset_details}.

\vspace{+0.1cm}
\noindent \textbf{Baselines.} We consider three categories of representative methods: (1) \textit{MAG models}: MMGCN~\cite{mmgcn}, MGAT~\cite{mgat}, LGMRec~\cite{lgmrec}, TMTE~\cite{zhu2026tmte}, and RoleMAG~\cite{zuo2026rolemag}; (2) \textit{GFMs}: GIT~\cite{wang2025git}, GFT~\cite{wang2024gft}, OpenGraph~\cite{xia2024opengraph}, and RiemannGFM~\cite{sun2025riemanngfm}; and (3) \textit{MGFMs}: UniGraph2~\cite{he2025unigraph2} and PLANET~\cite{liu2026planet}. Detailed descriptions and implementation settings for all baselines are provided in the Appendix~\ref{app:compared_baselines}.

\vspace{+0.1cm}
\noindent \textbf{Downstream Tasks.} We evaluate GraphBind on two task families: (1) \textit{Discriminative tasks}, including node classification and link prediction; and (2) \textit{Generative tasks}, including graph-to-text generation (G2Text) and graph-to-image generation (G2Image). Detailed task protocols and evaluation metrics are provided in the Appendix~\ref{app:experiment_setup}.

\begin{table*}
\centering
\caption{Model performance in discriminative tasks. The best results are $\textbf{bold}$. The second-best results are $\underline{underlined}$.}
\label{tab:Q1_1}
\renewcommand{\arraystretch}{1.2}
\resizebox{\textwidth}{!}{
\begin{tabular}{c|cccc|cccc|cccc}
\toprule

\multirow{3}{*}{\centering\textbf{Methods}}
& \multicolumn{4}{c|}{\textbf{Node Classification}}
& \multicolumn{4}{c|}{\textbf{Link Prediction}}
& \multicolumn{4}{c}{\textbf{Few-Shot Node Classification}} \\

\cmidrule{2-13}

& \multicolumn{2}{c}{Toys} 
& \multicolumn{2}{c|}{Grocery}
& \multicolumn{2}{c}{DY} 
& \multicolumn{2}{c|}{Bili\_Dance}
& \multicolumn{2}{c}{Movies-5Way}
& \multicolumn{2}{c}{RedditS-5Way} \\

& Acc & macro-F1
& Acc & macro-F1
& MRR & Hits@3
& MRR & Hits@3
& 5-Shot & 3-Shot
& 5-Shot & 3-Shot \\

\midrule

MMGCN
& $76.54_{\scriptsize \pm 0.25}$ & $73.61_{\scriptsize \pm 0.36}$ & $79.27_{\scriptsize \pm 0.20}$ & $68.78_{\scriptsize \pm 0.43}$
& $66.67_{\scriptsize \pm 0.24}$ & $69.62_{\scriptsize \pm 0.26}$ & $65.47_{\scriptsize \pm 0.57}$ & $70.72_{\scriptsize \pm 0.50}$
& $46.79_{\scriptsize \pm 4.22}$ & $42.42_{\scriptsize \pm 4.15}$ & $\underline{95.64}_{\scriptsize \pm 0.87}$ & $94.83_{\scriptsize \pm 0.71}$ \\

MGAT
& $76.82_{\scriptsize \pm 0.27}$ & $68.23_{\scriptsize \pm 0.52}$ & $80.29_{\scriptsize \pm 0.22}$ & $65.14_{\scriptsize \pm 0.54}$
& $65.45_{\scriptsize \pm 0.25}$ & $68.51_{\scriptsize \pm 0.25}$ & $65.48_{\scriptsize \pm 0.82}$ & $69.95_{\scriptsize \pm 0.36}$
& $47.45_{\scriptsize \pm 3.47}$ & $42.39_{\scriptsize \pm 3.42}$ & $94.48_{\scriptsize \pm 1.06}$ & $93.89_{\scriptsize \pm 0.57}$ \\

LGMRec
& $77.72_{\scriptsize \pm 0.21}$ & $75.60_{\scriptsize \pm 0.08}$ & $81.29_{\scriptsize \pm 0.03}$ & $70.85_{\scriptsize \pm 0.60}$
& $67.89_{\scriptsize \pm 0.12}$ & $72.61_{\scriptsize \pm 0.17}$ & $69.35_{\scriptsize \pm 0.26}$ & $75.73_{\scriptsize \pm 0.52}$
& $\underline{53.85}_{\scriptsize \pm 4.19}$ & $49.49_{\scriptsize \pm 4.30}$ & $95.60_{\scriptsize \pm 0.86}$ & $\underline{94.98}_{\scriptsize \pm 0.62}$ \\

TMTE
& $77.74_{\scriptsize \pm 0.25}$ & $70.60_{\scriptsize \pm 0.46}$ & $78.61_{\scriptsize \pm 0.27}$ & $60.40_{\scriptsize \pm 1.08}$
& $19.71_{\scriptsize \pm 1.80}$ & $19.97_{\scriptsize \pm 1.87}$ & $45.32_{\scriptsize \pm 1.41}$ & $58.83_{\scriptsize \pm 1.41}$
& $50.60_{\scriptsize \pm 4.15}$ & $46.78_{\scriptsize \pm 4.70}$ & $95.30_{\scriptsize \pm 0.90}$ & $94.40_{\scriptsize \pm 0.73}$ \\

RoleMAG
& $\textbf{78.85}_{\scriptsize \pm 0.32}$ & $\textbf{77.23}_{\scriptsize \pm 0.38}$ & $\underline{81.71}_{\scriptsize \pm 0.34}$ & $\underline{74.09}_{\scriptsize \pm 0.38}$
& $69.16_{\scriptsize \pm 1.04}$ & $75.92_{\scriptsize \pm 1.05}$ & $64.11_{\scriptsize \pm 1.25}$ & $75.85_{\scriptsize \pm 1.36}$
& $52.92_{\scriptsize \pm 2.54}$ & $46.72_{\scriptsize \pm 3.97}$ & $95.41_{\scriptsize \pm 0.80}$ & $94.64_{\scriptsize \pm 0.69}$ \\

\midrule

GIT
& $59.88_{\scriptsize \pm 0.43}$ & $55.31_{\scriptsize \pm 1.98}$ & $65.69_{\scriptsize \pm 0.86}$ & $52.07_{\scriptsize \pm 0.81}$
& $61.34_{\scriptsize \pm 0.88}$ & $66.39_{\scriptsize \pm 0.70}$ & $71.63_{\scriptsize \pm 1.00}$ & $74.12_{\scriptsize \pm 0.49}$
& $27.98_{\scriptsize \pm 2.70}$ & $26.58_{\scriptsize \pm 2.56}$ & $50.40_{\scriptsize \pm 4.13}$ & $47.69_{\scriptsize \pm 3.26}$ \\

GFT
& $76.45_{\scriptsize \pm 0.27}$ & $73.78_{\scriptsize \pm 0.35}$ & $80.18_{\scriptsize \pm 0.48}$ & $70.65_{\scriptsize \pm 0.76}$
& $\underline{72.79}_{\scriptsize \pm 0.35}$ & $\textbf{76.30}_{\scriptsize \pm 0.07}$ & $\underline{72.01}_{\scriptsize \pm 0.41}$ & $\underline{77.86}_{\scriptsize \pm 0.31}$
& $44.36_{\scriptsize \pm 3.06}$ & $40.69_{\scriptsize \pm 3.09}$ & $90.58_{\scriptsize \pm 1.60}$ & $89.15_{\scriptsize \pm 1.66}$ \\

OpenGraph
& $54.80_{\scriptsize \pm 0.56}$ & $45.52_{\scriptsize \pm 1.01}$ & $67.97_{\scriptsize \pm 1.03}$ & $49.03_{\scriptsize \pm 2.27}$
& $58.22_{\scriptsize \pm 0.45}$ & $65.31_{\scriptsize \pm 0.44}$ & $54.72_{\scriptsize \pm 0.35}$ & $64.80_{\scriptsize \pm 0.41}$
& $41.67_{\scriptsize \pm 4.59}$ & $38.80_{\scriptsize \pm 3.59}$ & $76.92_{\scriptsize \pm 5.11}$ & $75.25_{\scriptsize \pm 5.01}$ \\

RiemannGFM
& $77.39_{\scriptsize \pm 0.19}$ & $75.74_{\scriptsize \pm 0.19}$
& $81.16_{\scriptsize \pm 0.32}$ & $71.93_{\scriptsize \pm 0.16}$
& $27.30_{\scriptsize \pm 0.75}$ & $31.31_{\scriptsize \pm 1.07}$
& $39.45_{\scriptsize \pm 0.44}$ & $48.20_{\scriptsize \pm 0.82}$
& $43.97_{\scriptsize \pm 4.01}$ & $39.84_{\scriptsize \pm 4.61}$
& $94.83_{\scriptsize \pm 0.72}$ & $94.22_{\scriptsize \pm 0.67}$ \\

\midrule

UniGraph2
& $73.66_{\scriptsize \pm 0.89}$ & $67.25_{\scriptsize \pm 2.24}$
& $70.32_{\scriptsize \pm 2.39}$ & $51.36_{\scriptsize \pm 4.65}$
& $63.87_{\scriptsize \pm 0.83}$ & $70.64_{\scriptsize \pm 0.98}$
& $50.60_{\scriptsize \pm 7.37}$ & $59.01_{\scriptsize \pm 9.25}$
& $52.40_{\scriptsize \pm 2.90}$ & $\underline{49.67}_{\scriptsize \pm 3.39}$
& $84.53_{\scriptsize \pm 5.29}$ & $83.67_{\scriptsize \pm 5.52}$ \\

PLANET
& $77.57_{\scriptsize \pm 0.38}$ & $75.13_{\scriptsize \pm 0.32}$
& $75.54_{\scriptsize \pm 2.26}$ & $59.94_{\scriptsize \pm 5.49}$
& $68.75_{\scriptsize \pm 0.37}$ & $71.78_{\scriptsize \pm 0.27}$
& $70.28_{\scriptsize \pm 0.95}$ & $73.86_{\scriptsize \pm 0.62}$
& $42.23_{\scriptsize \pm 5.93}$ & $44.63_{\scriptsize \pm 2.03}$
& $65.77_{\scriptsize \pm 6.97}$ & $64.33_{\scriptsize \pm 6.34}$ \\

GraphBind
& $\underline{78.08}_{\scriptsize \pm 0.05}$ & $\underline{77.06}_{\scriptsize \pm 0.04}$
& $\textbf{81.86}_{\scriptsize \pm 0.12}$ & $\textbf{74.58}_{\scriptsize \pm 0.14}$
& $\textbf{75.23}_{\scriptsize \pm 0.70}$ & $\underline{76.00}_{\scriptsize \pm 0.41}$
& $\textbf{75.88}_{\scriptsize \pm 0.37}$ & $\textbf{79.36}_{\scriptsize \pm 0.32}$
& $\textbf{56.44}_{\scriptsize \pm 1.69}$ & $\textbf{52.67}_{\scriptsize \pm 1.88}$
& $\textbf{96.37}_{\scriptsize \pm 0.75}$ & $\textbf{95.34}_{\scriptsize \pm 0.67}$ \\

\bottomrule
\end{tabular}
}
\end{table*}

\subsection{Main Results (Answer for Q1)}
\label{sec:main_results}

\noindent \textbf{To answer Q1}, we compare GraphBind with representative MAG models, GFMs, and MGFMs on discriminative and generative tasks in Tables~\ref{tab:Q1_1} and~\ref{tab:Q1_2}.

\vspace{+0.1cm}
\noindent \textbf{Discriminative Tasks.} Under supervised tasks, GraphBind achieves the best results on Grocery and Bili\_Dance and the highest MRR on DY, while ranking second on Toys and in Hits@3 on DY. It also achieves the best performance across all four few-shot settings, improving accuracy by 2.59 percentage points over the strongest baseline on Movies-5Way 5-shot. These results demonstrate that binding structural and semantic information into a global shared space yields task-relevant representations that transfer effectively across supervision regimes. Furthermore, the topology prefix prompter provides lightweight task-specific adaptation to this space, enabling GraphBind to specialize for different discriminative objectives.

\vspace{+0.1cm}
\noindent \textbf{Generative Tasks.} For G2Image, GraphBind achieves 70.01 CLIP-S and 48.91 DINO-S, delivering strong visual semantic alignment, while for G2Text, it obtains 16.43 BLEU-4 and 43.45 ROUGE-L, surpassing the strongest baseline by 3.60 and 4.50 points, respectively. These results indicate that the shared representation provides complementary node semantics and neighborhood information for generation. Furthermore, the G2Text gains indicate that the G2M-Adapter enables generators to effectively use this graph context through lightweight adaptation.

\subsection{Ablation Study (Answer for Q2)}
\label{sec:ablation_study}

\noindent \textbf{To address Q2}, we compare GraphBind with 5 variants to examine its core architecture, which couples reliable neighborhood semantic completion with topology guided shared latent calibration. Since these stages are coupled, each variant replaces them together under the same experimental protocol. We organize the variants into three groups: (1) \textit{Simple Propagation}: We apply standard GNN encoders \textit{GCN}~\cite{GCN}, \textit{GAT}~\cite{GAT}, and \textit{GraphSAGE}~\cite{GraphSAGE} over the original adjacency, retaining neighborhood propagation without neighbor reliability estimation, the self retention gate, or topology conditioned calibration; (2) \textit{Topology Encoding (T.E.)}: We apply a node-wise semantic transformation using an MLP with both node semantics and topology encodings, without neighborhood propagation; and (3) \textit{Node Only (N.O.)}: We remove topology encodings and the entire graph aware path, directly using the initialized node semantics.

\begin{table*}[t!]
    \centering
    \begin{minipage}[t]{0.485\linewidth}
        \centering
        \captionsetup{width=\linewidth}
        \captionof{table}{Model performance in generative tasks. The best results are $\textbf{bold}$. The second results are $\underline{underlined}$.}
        \label{tab:Q1_2}
        \renewcommand{\arraystretch}{1.2}
        \setlength{\tabcolsep}{2.0pt}
        \scriptsize
        \resizebox{\linewidth}{1.9cm}{%
        \begin{tabular}{c|cc|cc}
        \midrule[0.3pt]
        Methods    & \multicolumn{2}{c|}{G2Image} & \multicolumn{2}{c}{G2Text}    \\
          & CLIP-S       & DINO-S   & BLEU-4         & ROUGE-L    \\ \midrule[0.3pt]
        MMGCN      & $69.26_{\scriptsize \pm 0.06}$    & $46.60_{\scriptsize \pm 0.03}$   & $7.36_{\scriptsize \pm 0.12}$      & $30.85_{\scriptsize \pm 0.10}$    \\
        MGAT    & $67.65_{\scriptsize \pm 0.01}$    & $45.23_{\scriptsize \pm 0.11}$   & $7.30_{\scriptsize \pm 0.05}$      & $30.86_{\scriptsize \pm 0.15}$    \\
        LGMRec   & $68.69_{\scriptsize \pm 0.07}$    & $47.25_{\scriptsize \pm 0.05}$   & $7.54_{\scriptsize \pm 0.07}$      & $31.42_{\scriptsize \pm 0.08}$    \\ \midrule[0.3pt]
        UniGraph2      & $\underline{69.78}_{\scriptsize \pm 0.06}$    & $\textbf{49.22}_{\scriptsize \pm 0.08}$   & $5.57_{\scriptsize \pm 0.25}$      & $26.87_{\scriptsize \pm 0.65}$    \\
        PLANET    & $66.76_{\scriptsize \pm 0.03}$    & $45.32_{\scriptsize \pm 0.16}$   & $\underline{12.83}_{\scriptsize \pm 0.73}$      & $\underline{38.95}_{\scriptsize \pm 1.10}$    \\
        GraphBind  & $\textbf{70.01}_{\scriptsize \pm 0.02}$    & $\underline{48.91}_{\scriptsize \pm 0.14}$   & $\textbf{16.43}_{\scriptsize \pm 0.36}$       & $\textbf{43.45}_{\scriptsize \pm 0.10}$    \\ 
        \midrule[0.3pt]
        \end{tabular}
        }
    \end{minipage}
    \hfill
    \begin{minipage}[t]{0.485\linewidth}
        \centering
        \captionsetup{width=\linewidth}
        \captionof{table}{Ablation results on Toys, DY, SemArt, and Flickr30k for the 4 tasks, respectively. The best results are $\textbf{bold}$.}
        \label{tab:Q2}
        \renewcommand{\arraystretch}{1.2}
        \setlength{\tabcolsep}{2.0pt}
        \scriptsize
        \resizebox{\linewidth}{1.9cm}{%
        \begin{tabular}{c|cccc}
        \midrule[0.3pt]
        Variant & \multicolumn{4}{c}{Downstream Tasks} \\
        & NC & LP & G2Image & G2Text \\ \midrule[0.3pt]
        GCN & $75.27_{\scriptsize \pm 0.12}$ & $74.91_{\scriptsize \pm 1.10}$ & $64.48_{\scriptsize \pm 0.62}$ & $13.05_{\scriptsize \pm 0.05}$ \\
        GAT & $75.79_{\scriptsize \pm 0.73}$ & $74.70_{\scriptsize \pm 1.03}$ & $65.66_{\scriptsize \pm 2.37}$ & $12.86_{\scriptsize \pm 0.07}$ \\
        SAGE & $75.83_{\scriptsize \pm 0.66}$ & $75.19_{\scriptsize \pm 1.12}$ & $62.75_{\scriptsize \pm 1.44}$ & $12.91_{\scriptsize \pm 0.03}$ \\
        T.E. & $73.22_{\scriptsize \pm 1.73}$ & $66.39_{\scriptsize \pm 1.29}$ & $64.25_{\scriptsize \pm 1.79}$ & $14.39_{\scriptsize \pm 0.08}$ \\
        N.O. & $73.33_{\scriptsize \pm 1.62}$ & $69.73_{\scriptsize \pm 1.12}$ & $66.65_{\scriptsize \pm 1.72}$ & $14.67_{\scriptsize \pm 0.23}$ \\ \midrule[0.3pt]
        GraphBind & $\textbf{77.06}_{\scriptsize \pm 0.04}$ & $\textbf{75.23}_{\scriptsize \pm 0.70}$ & $\textbf{70.01}_{\scriptsize \pm 0.02}$ & $\textbf{16.43}_{\scriptsize \pm 0.36}$ \\
        \midrule[0.3pt]
        \end{tabular}
        }
    \end{minipage}
\end{table*}

As shown in Table~\ref{tab:Q2}, GraphBind achieves the best result on all four tasks, with larger gains over the strongest variants on the two generative tasks. Among the variants, \textit{Simple Propagation} performs better on Toys and DY, while \textit{N.O.} performs best on SemArt and Flickr30k. This task-dependent ordering indicates that direct propagation over the original adjacency provides useful information for discriminative graph tasks and can introduce unsuitable context for generation. Moreover, \textit{T.E.} performs below \textit{N.O.} on every dataset, showing that adding a node-wise transformation and topology encoding without neighborhood propagation does not improve upon self semantics alone. Overall, the proposed topology-guided encoder outperforms the tested alternatives across all four tasks, supporting its effectiveness in learning shared graph representations.

\subsection{Robustness Analysis (Answer for Q3)}
\label{sec:robustness_analysis}

\begin{figure*}[t]
    \centering
    \includegraphics[width=1.0\linewidth]{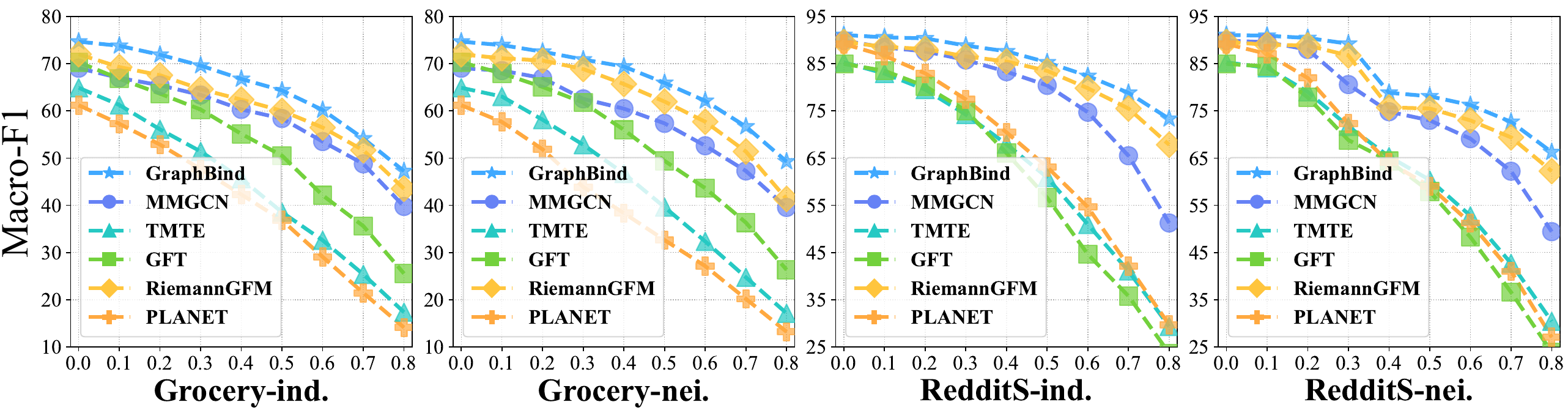}
    \caption{Test macro-F1 under independent and neighborhood modality absence on Grocery and RedditS. Missing ratios range from 0.0 to 0.8 in steps of 0.1.}
    \label{fig:q3}
\end{figure*}

\noindent \textbf{To answer Q3}, we evaluate robustness to modality absence on Grocery and RedditS by masking modality inputs at inference. Specifically, we impose two noise patterns to model distinct failure processes. \textit{Independent} noise masks each node's modalities separately and simulates dispersed random failures. \textit{Neighborhood} noise masks locally concentrated nodes and simulates correlated failures within a community or collection context. We increase the missing ratio from 0.0 to 0.8 in steps of 0.1 and report macro-F1. As shown in Fig.~\ref{fig:q3}, performance declines for all methods as modality absence increases. Additionally, GraphBind exhibits the smoothest degradation and remains the leading method across both datasets and missing patterns. Under neighborhood noise, GraphBind retains leading performance even when correlated modality failures remove locally shared information, showing that it remains effective when modality absence is concentrated in neighboring nodes. Together, the results support the robustness of GraphBind to increasing modality absence under noisy settings.

\subsection{Efficiency Analysis (Answer for Q4)}
\label{sec:efficiency_analysis}

\begin{table}[t]
  \centering
  \caption{Computational efficiency of pretraining and node classification adaptation.}
  \label{tab:Q4}
  \renewcommand{\arraystretch}{1.1}
  \setlength{\tabcolsep}{2.8pt}
  \small
  \resizebox{\textwidth}{!}{%
  \begin{tabular}{c|ccccc}
    \toprule
    \textbf{Methods} & \textbf{PT (s/epoch)} & \textbf{FT (s/epoch)} & \textbf{IT (ms/1k nodes)} & \textbf{Train. Para. (M)} & \textbf{Total Para. (M)} \\
    \midrule
    MMGCN      & --       & $0.0319$ & $12.19$ & $3.3503$   & $3.3503$   \\
    MGAT       & --       & $0.0389$ & $30.63$ & $2.5779$   & $2.5779$   \\
    LGMRec     & --       & $0.0359$ & $36.66$ & $0.2405$   & $0.2405$   \\
    TMTE       & --       & $0.1906$ & $40.66$ & $0.1180$   & $0.1180$   \\
    RoleMAG    & --       & $0.4534$ & $73.42$ & $3.6890$   & $3.6890$   \\
    \midrule
    GIT        & $0.2755$ & $0.0340$ & $15.19$ & $0.9905$   & $0.9905$   \\
    GFT        & $0.2344$ & $0.0893$ & $34.27$ & $2.8273$   & $2.8273$   \\
    OpenGraph  & $0.1343$ & $0.0950$ & $39.64$ & $1.1884$   & $1.1884$   \\
    RiemannGFM & $1.3383$ & $0.1891$ & $33.26$ & $0.2157$   & $0.4717$   \\
    \midrule
    UniGraph2  & $1.0177$ & $0.1596$ & $53.83$ & $20.7237$  & $20.7237$  \\
    PLANET     & $1.6948$ & $3.6859$ & $92.39$ & $221.9436$ & $221.9436$ \\
    GraphBind  & $0.3803$ & $0.4022$ & $59.17$ & $5.3215$   & $23.0692$  \\
    \bottomrule
  \end{tabular}
  }
\end{table}

\noindent \textbf{To answer Q4}, we compare the computational costs of GraphBind with representative baselines in Table~\ref{tab:Q4}. GraphBind's finetuning and inference costs are higher than those of most lightweight MAG baselines, whereas its pretraining is faster than that of both MGFMs and its inference is faster than PLANET; this stage-dependent pattern indicates that GraphBind reduces repeated pretraining cost while incurring additional downstream cost for processing richer graph context. Moreover, among MGFMs, GraphBind requires substantially fewer trainable parameters, updating 5.32M parameters compared with 20.72M for UniGraph2 and 221.94M for PLANET, respectively. It also achieves the lowest proportion of trainable parameters among all compared methods. Combined with its strong classification performance, these findings demonstrate the effectiveness and parameter efficiency of GraphBind's lightweight finetuning design.

\section{Conclusion}
\label{sec:conclusion}

In this paper, we propose GraphBind, a multimodal graph foundation model that uses graph topology to bind heterogeneous modality semantics into a unified representation space. We revisit the role of topology in multimodal binding and show that existing methods largely treat it as a propagation or alignment signal, underexploiting its capacity to organize shared semantics. GraphBind couples topology-guided pretraining with lightweight fine-tuning, learning a shared space that integrates structure and semantics for both discriminative and generative tasks. Extensive experiments against 11 representative baselines show that GraphBind achieves strong performance across both discriminative and generative tasks, demonstrating strong generalization and robustness. We believe this work provides a valuable foundation for topology-driven multimodal graph foundation models and motivates further exploration of scalable graph-semantic binding.

\section*{AI Usage Statement}
No generative AI tools were used for tasks requiring disclosure under the ICLR 2027 policy. The authors take full responsibility for the content of this paper.

\bibliographystyle{iclr2027_conference}
\bibliography{citation}

@article{liu2023towards,
  title = {Towards Graph Foundation Models: A Survey and Beyond},
  author = {Liu, Jiawei and Yang, Cheng and Lu, Zhiyuan and Chen, Junze and Li, Yibo and Zhang, Mengmei and Bai, Ting and Fang, Yuan and Sun, Lichao and Yu, Philip S. and Shi, Chuan},
  journal = {arXiv preprint arXiv:2310.11829},
  year = {2023}
}

@article{wang2025graphfoundationcomprehensive,
  title = {Graph Foundation Models: A Comprehensive Survey},
  author = {Wang, Zehong and Liu, Zheyuan and Ma, Tianyi and Li, Jiazheng and Zhang, Zheyuan and Fu, Xingbo and Li, Yiyang and Yuan, Zhengqing and Song, Wei and Ma, Yijun and Zeng, Qingkai and Chen, Xiusi and Zhao, Jianan and Li, Jundong and Jiang, Meng and Li{\`o}, Pietro and Chawla, Nitesh and Zhang, Chuxu and Ye, Yanfang},
  journal = {arXiv preprint arXiv:2505.15116},
  year = {2025}
}

@inproceedings{zhu2025mosaic,
  title = {Mosaic of Modalities: A Comprehensive Benchmark for Multimodal Graph Learning},
  author = {Zhu, Jing and Zhou, Yuhang and Qian, Shengyi and He, Zhongmou and Zhao, Tong and Shah, Neil and Koutra, Danai},
  booktitle = {Proceedings of the IEEE/CVF Conference on Computer Vision and Pattern Recognition},
  pages = {14215--14224},
  year = {2025},
  doi = {10.1109/CVPR52734.2025.01326}
}

@inproceedings{wan2026openmag,
  title = {{OpenMAG}: A Comprehensive Benchmark for Multimodal-Attributed Graph},
  author = {Wan, Chenxi and Li, Xunkai and Zuo, Yilong and Deng, Haokun and Li, Sihan and Fan, Bowen and Qin, Hongchao and Li, Ronghua and Wang, Guoren},
  booktitle = {Proceedings of the International Conference on Machine Learning},
  year = {2026}
}

@inproceedings{he2025unigraph2,
  title = {{UniGraph2}: Learning a Unified Embedding Space to Bind Multimodal Graphs},
  author = {He, Yufei and Sui, Yuan and He, Xiaoxin and Liu, Yue and Sun, Yifei and Hooi, Bryan},
  booktitle = {Proceedings of the ACM Web Conference},
  pages = {1759--1770},
  doi = {10.1145/3696410.3714818},
  year = {2025}
}

@inproceedings{liu2024oneforall,
  title = {One for All: Towards Training One Graph Model for All Classification Tasks},
  author = {Liu, Hao and Feng, Jiarui and Kong, Lecheng and Liang, Ningyue and Tao, Dacheng and Chen, Yixin and Zhang, Muhan},
  booktitle = {International Conference on Learning Representations},
  year = {2024}
}

@inproceedings{tang2024graphgpt,
  title = {GraphGPT: Graph Instruction Tuning for Large Language Models},
  author = {Tang, Jiabin and Yang, Yuhao and Wei, Wei and Shi, Lei and Su, Lixin and Cheng, Suqi and Yin, Dawei and Huang, Chao},
  booktitle = {Proceedings of the 47th International ACM SIGIR Conference on Research and Development in Information Retrieval},
  year = {2024}
}

@inproceedings{liu2026planet,
  title = {Toward Effective Multimodal Graph Foundation Model: A Divide-and-Conquer Based Approach},
  author = {Liu, Sicheng and Li, Xunkai and Su, Daohan and Zhang, Ru and Qin, Hongchao and Li, Ronghua and Wang, Guoren},
  booktitle = {Proceedings of the International Conference on Machine Learning},
  year = {2026}
}

@inproceedings{sun2026mario,
  title = {{Mario}: Multimodal Graph Reasoning with Large Language Models},
  author = {Sun, Yuanfu and Li, Kang and Guo, Pengkang and Liu, Jiajin and Tan, Qiaoyu},
  booktitle = {Proceedings of the IEEE/CVF Conference on Computer Vision and Pattern Recognition},
  pages = {19219--19228},
  year = {2026}
}

@article{wu2026missingmodality,
  title = {Deep Multimodal Learning with Missing Modality: A Survey},
  author = {Wu, Renjie and Wang, Hu and Chen, Hsiang-Ting and Carneiro, Gustavo},
  journal = {Transactions on Machine Learning Research},
  year = {2026}
}

@inproceedings{wang2023shaspec,
  title = {Multi-Modal Learning with Missing Modality via Shared-Specific Feature Modelling},
  author = {Wang, Hu and Chen, Yuanhong and Ma, Congbo and Avery, Jodie and Hull, Louise and Carneiro, Gustavo},
  booktitle = {Proceedings of the IEEE/CVF Conference on Computer Vision and Pattern Recognition},
  year = {2023}
}

@inproceedings{radford2021clip,
  title = {Learning Transferable Visual Models From Natural Language Supervision},
  author = {Radford, Alec and Kim, Jong Wook and Hallacy, Chris and Ramesh, Aditya and Goh, Gabriel and Agarwal, Sandhini and Sastry, Girish and Askell, Amanda and Mishkin, Pamela and Clark, Jack and Krueger, Gretchen and Sutskever, Ilya},
  booktitle = {Proceedings of the 38th International Conference on Machine Learning},
  pages = {8748--8763},
  year = {2021}
}

@inproceedings{girdhar2023imagebind,
  title = {ImageBind: One Embedding Space To Bind Them All},
  author = {Girdhar, Rohit and El-Nouby, Alaaeldin and Liu, Zhuang and Singh, Mannat and Alwala, Kalyan Vasudev and Joulin, Armand and Misra, Ishan},
  booktitle = {Proceedings of the IEEE/CVF Conference on Computer Vision and Pattern Recognition},
  year = {2023}
}

@inproceedings{liu2023graphprompt,
  title = {GraphPrompt: Unifying Pre-Training and Downstream Tasks for Graph Neural Networks},
  author = {Liu, Zemin and Yu, Xingtong and Fang, Yuan and Zhang, Xinming},
  booktitle = {Proceedings of the ACM Web Conference},
  year = {2023}
}

@inproceedings{sun2023allinone,
  title = {All in One: Multi-Task Prompting for Graph Neural Networks},
  author = {Sun, Xiangguo and Cheng, Hong and Li, Jia and Liu, Bo and Guan, Jihong},
  booktitle = {Proceedings of the ACM SIGKDD Conference on Knowledge Discovery and Data Mining},
  year = {2023}
}

@inproceedings{hou2022graphmae,
  title = {GraphMAE: Self-Supervised Masked Graph Autoencoders},
  author = {Hou, Zhenyu and Liu, Xiao and Cen, Yukuo and Dong, Yuxiao and Yang, Hongxia and Wang, Chunjie and Tang, Jie},
  booktitle = {Proceedings of the ACM SIGKDD Conference on Knowledge Discovery and Data Mining},
  year = {2022}
}

@inproceedings{zhang2021ccassg,
  title = {From Canonical Correlation Analysis to Self-supervised Graph Neural Networks},
  author = {Zhang, Hengrui and Wu, Qitian and Yan, Junchi and Wipf, David and Yu, Philip S.},
  booktitle = {Advances in Neural Information Processing Systems},
  year = {2021}
}

@article{zhou2022coop,
  title = {Learning to Prompt for Vision-Language Models},
  author = {Zhou, Kaiyang and Yang, Jingkang and Loy, Chen Change and Liu, Ziwei},
  journal = {International Journal of Computer Vision},
  year = {2022}
}

@inproceedings{li2023blip2,
  title = {BLIP-2: Bootstrapping Language-Image Pre-training with Frozen Image Encoders and Large Language Models},
  author = {Li, Junnan and Li, Dongxu and Savarese, Silvio and Hoi, Steven},
  booktitle = {Proceedings of the 40th International Conference on Machine Learning},
  year = {2023}
}

@inproceedings{Amazon2018,
  title = {Justifying Recommendations Using Distantly-Labeled Reviews and Fine-Grained Aspects},
  author = {Ni, Jianmo and Li, Jiacheng and McAuley, Julian J.},
  booktitle = {Proceedings of the Conference on Empirical Methods in Natural Language Processing and the International Joint Conference on Natural Language Processing},
  pages = {188--197},
  year = {2019}
}

@inproceedings{hou2024bridging,
  title = {Bridging Language and Items for Retrieval and Recommendation: Benchmarking {LLMs} as Semantic Encoders},
  author = {Hou, Yupeng and Li, Jiacheng and Fu, Xiangjun and He, Zhankui and Yan, An and Chen, Xiusi and McAuley, Julian},
  booktitle = {Proceedings of the 64th Annual Meeting of the Association for Computational Linguistics (Volume 1: Long Papers)},
  pages = {3251--3265},
  doi = {10.18653/v1/2026.acl-long.147},
  year = {2026}
}

@inproceedings{RedditS,
  title = {RedCaps: Web-curated Image-Text Data Created by the People, for the People},
  author = {Desai, Karan and Kaul, Gaurav and Aysola, Zubin and Johnson, Justin},
  booktitle = {Proceedings of the Neural Information Processing Systems Track on Datasets and Benchmarks},
  volume = {1},
  year = {2021}
}

@article{ninerec,
  title = {NineRec: A Benchmark Dataset Suite for Evaluating Transferable Recommendation},
  author = {Zhang, Jiaqi and Cheng, Yu and Ni, Yongxin and Pan, Yunzhu and Yuan, Zheng and Fu, Junchen and Li, Youhua and Wang, Jie and Yuan, Fajie},
  journal = {IEEE Transactions on Pattern Analysis and Machine Intelligence},
  year = {2024}
}

@inproceedings{flickr30k,
  title = {Flickr30k Entities: Collecting Region-to-Phrase Correspondences for Richer Image-to-Sentence Models},
  author = {Plummer, Bryan A. and Wang, Liwei and Cervantes, Chris M. and Caicedo, Juan C. and Hockenmaier, Julia and Lazebnik, Svetlana},
  booktitle = {Proceedings of the IEEE International Conference on Computer Vision},
  pages = {2641--2649},
  year = {2015}
}

@inproceedings{semart,
  title = {How to Read Paintings: Semantic Art Understanding with Multi-modal Retrieval},
  author = {Garcia, Noa and Vogiatzis, George},
  booktitle = {European Conference on Computer Vision Workshops},
  pages = {676--691},
  year = {2018}
}

@inproceedings{mmgcn,
  title = {MMGCN: Multi-modal Graph Convolution Network for Personalized Recommendation of Micro-video},
  author = {Wei, Yinwei and Wang, Xiang and Nie, Liqiang and He, Xiangnan and Hong, Richang and Chua, Tat-Seng},
  booktitle = {Proceedings of the ACM International Conference on Multimedia},
  pages = {1437--1445},
  year = {2019}
}

@article{mgat,
  title = {MGAT: Multimodal Graph Attention Network for Recommendation},
  author = {Tao, Zhulin and Wei, Yinwei and Wang, Xiang and He, Xiangnan and Huang, Xianglin and Chua, Tat-Seng},
  journal = {Information Processing \& Management},
  volume = {57},
  number = {5},
  pages = {102277},
  year = {2020}
}

@inproceedings{lgmrec,
  title = {LGMRec: Local and Global Graph Learning for Multimodal Recommendation},
  author = {Guo, Zhiqiang and Li, Jianjun and Li, Guohui and Wang, Chaoyang and Shi, Si and Ruan, Bin},
  booktitle = {Proceedings of the AAAI Conference on Artificial Intelligence},
  pages = {8454--8462},
  year = {2024}
}

@inproceedings{zhu2026tmte,
  title = {{TMTE}: Effective Multimodal Graph Learning with Task-aware Modality and Topology Co-evolution},
  author = {Zhu, Yinlin and Li, Xunkai and Wu, Di and Luo, Wang and Hu, Miao and Quan, Guocong},
  booktitle = {Proceedings of the ACM International Conference on Multimedia},
  year = {2026}
}

@inproceedings{zuo2026rolemag,
  title = {{RoleMAG}: Learning Neighbor Roles in Multimodal Graphs},
  author = {Zuo, Yilong and Li, Xunkai and Zhang, Zhihan and Li, Ronghua and Wang, Guoren},
  booktitle = {Proceedings of the ACM International Conference on Multimedia},
  year = {2026}
}

@inproceedings{wang2025git,
  title = {Towards Graph Foundation Models: Learning Generalities Across Graphs via Task-Trees},
  author = {Wang, Zehong and Zhang, Zheyuan and Ma, Tianyi and Chawla, Nitesh V. and Zhang, Chuxu and Ye, Yanfang},
  booktitle = {Proceedings of the International Conference on Machine Learning},
  year = {2025}
}

@inproceedings{wang2024gft,
  title = {GFT: Graph Foundation Model with Transferable Tree Vocabulary},
  author = {Wang, Zehong and Zhang, Zheyuan and Chawla, Nitesh and Zhang, Chuxu and Ye, Yanfang},
  booktitle = {Advances in Neural Information Processing Systems},
  year = {2024}
}

@inproceedings{xia2024opengraph,
  title = {OpenGraph: Towards Open Graph Foundation Models},
  author = {Xia, Lianghao and Kao, Ben and Huang, Chao},
  booktitle = {Findings of the Association for Computational Linguistics: EMNLP},
  year = {2024}
}

@inproceedings{sun2025riemanngfm,
  title = {RiemannGFM: Learning a Graph Foundation Model from Riemannian Geometry},
  author = {Sun, Li and Huang, Zhenhao and Zhou, Suyang and Wan, Qiqi and Peng, Hao and Yu, Philip S.},
  booktitle = {Proceedings of the ACM Web Conference},
  pages = {1154--1165},
  year = {2025}
}

@inproceedings{loshchilov2019decoupled,
  title = {Decoupled Weight Decay Regularization},
  author = {Loshchilov, Ilya and Hutter, Frank},
  booktitle = {International Conference on Learning Representations},
  year = {2019}
}

@article{zhang2022opt,
  title = {OPT: Open Pre-trained Transformer Language Models},
  author = {Zhang, Susan and Roller, Stephen and Goyal, Naman and Artetxe, Mikel and Chen, Moya and others},
  journal = {arXiv preprint arXiv:2205.01068},
  year = {2022}
}

@inproceedings{rombach2022high,
  title = {High-Resolution Image Synthesis with Latent Diffusion Models},
  author = {Rombach, Robin and Blattmann, Andreas and Lorenz, Dominik and Esser, Patrick and Ommer, Bjorn},
  booktitle = {Proceedings of the IEEE/CVF Conference on Computer Vision and Pattern Recognition},
  year = {2022}
}

@article{oquab2023dinov2,
  title = {{DINOv2}: Learning Robust Visual Features without Supervision},
  author = {Oquab, Maxime and Darcet, Timoth{\'e}e and others},
  journal = {Transactions on Machine Learning Research},
  year = {2024}
}

@inproceedings{GCN,
  author       = {Thomas N. Kipf and
                  Max Welling},
  title        = {Semi-Supervised Classification with Graph Convolutional Networks},
  booktitle    = {{ICLR} (Poster)},
  publisher    = {OpenReview.net},
  year         = {2017}
}

@inproceedings{
GAT,
title={Graph Attention Networks},
author={Petar Veličković and Guillem Cucurull and Arantxa Casanova and Adriana Romero and Pietro Liò and Yoshua Bengio},
booktitle={International Conference on Learning Representations},
year={2018}
}

@inproceedings{GraphSAGE,
  author       = {William L. Hamilton and
                  Zhitao Ying and
                  Jure Leskovec},
  title        = {Inductive Representation Learning on Large Graphs},
  booktitle    = {{NIPS}},
  pages        = {1024--1034},
  year         = {2017}
}

@article{dwivedi2023benchmarking,
  title = {Benchmarking Graph Neural Networks},
  author = {Dwivedi, Vijay Prakash and Joshi, Chaitanya K. and Luu, Anh Tuan and Laurent, Thomas and Bengio, Yoshua and Bresson, Xavier},
  journal = {Journal of Machine Learning Research},
  volume = {24},
  number = {43},
  pages = {1--48},
  year = {2023}
}

@inproceedings{dwivedi2022lspe,
  title = {Graph Neural Networks with Learnable Structural and Positional Representations},
  author = {Dwivedi, Vijay Prakash and Luu, Anh Tuan and Laurent, Thomas and Bengio, Yoshua and Bresson, Xavier},
  booktitle = {International Conference on Learning Representations},
  year = {2022}
}

@inproceedings{hu2022lora,
  title = {{LoRA}: Low-Rank Adaptation of Large Language Models},
  author = {Hu, Edward J. and Shen, Yelong and Wallis, Phillip and Allen-Zhu, Zeyuan and Li, Yuanzhi and Wang, Shean and Wang, Lu and Chen, Weizhu},
  booktitle = {International Conference on Learning Representations},
  year = {2022}
}

@inproceedings{papineni2002bleu,
  title = {{Bleu}: a Method for Automatic Evaluation of Machine Translation},
  author = {Papineni, Kishore and Roukos, Salim and Ward, Todd and Zhu, Wei-Jing},
  booktitle = {Proceedings of the 40th Annual Meeting of the Association for Computational Linguistics},
  pages = {311--318},
  doi = {10.3115/1073083.1073135},
  year = {2002}
}

@inproceedings{lin2004rouge,
  title = {{ROUGE}: A Package for Automatic Evaluation of Summaries},
  author = {Lin, Chin-Yew},
  booktitle = {Text Summarization Branches Out},
  pages = {74--81},
  year = {2004}
}

@inproceedings{vedantam2015cider,
  title = {{CIDEr}: Consensus-Based Image Description Evaluation},
  author = {Vedantam, Ramakrishna and Zitnick, C. Lawrence and Parikh, Devi},
  booktitle = {Proceedings of the IEEE Conference on Computer Vision and Pattern Recognition},
  pages = {4566--4575},
  year = {2015}
}

\appendix
\section{Appendix}

\subsection{Theorem Proofs}
\label{app:theorem_proof}

\noindent\textbf{Proof of Theorem~\ref{thm:reliable_neighbor}.}
Fix a node $i$ and compare the two completion inputs specified in the theorem. We trace the perturbation through the neighborhood aggregation and the completion gate.

\noindent\textit{Fixed quantities and perturbation source.}
Write $\mathbf{y}_{ij}=\mathbf{W}_n\mathbf{h}_j^{(0)}$ and $\mathbf{y}_{ij}'=\mathbf{W}_n\mathbf{h}_j^{(0)\prime}$ for the transformed messages received from neighbor $j$. The two inputs share the center-node state, topology encoding, and reliability scores. The completion gate depends only on these quantities through the center-node state, topology encoding, and mean neighbor reliability, so it is identical for both inputs. Their normalization constant and aggregation weights are therefore also shared:
\begin{equation}
\begin{aligned}
Z_i
&=
\sum_{k\in\mathcal{N}(i)}r_{ik}+\epsilon,\\
\alpha_{ij}
&=
\alpha_{ij}'
=
\frac{r_{ij}}{Z_i}.
\end{aligned}
\label{eq:app_fixed_weights}
\end{equation}
The reliable messages remain unchanged, so $\mathbf{y}_{ij}=\mathbf{y}_{ij}'$ for $j\in R(i)$. All nonzero message differences are confined to $U(i)$, where $\|\mathbf{y}_{ij}-\mathbf{y}_{ij}'\|_2\le\delta$.

\noindent\textit{Perturbation of the neighborhood message.}
Let $\mathbf{m}_i$ and $\mathbf{m}_i'$ denote the two neighborhood messages in Eq.~\ref{eq:semantic_completion}. Subtracting their weighted sums cancels every reliable-neighbor term. The remaining terms satisfy
\begin{equation}
\begin{aligned}
\|\mathbf{m}_i-\mathbf{m}_i'\|_2
&=
\left\|
\sum_{j\in U(i)}\alpha_{ij}
(\mathbf{y}_{ij}-\mathbf{y}_{ij}')
\right\|_2\\
&\le
\sum_{j\in U(i)}\alpha_{ij}
\|\mathbf{y}_{ij}-\mathbf{y}_{ij}'\|_2\\
&\le
\frac{U_i^w}{R_i^w+U_i^w+\epsilon}\,\delta,
\end{aligned}
\label{eq:app_message_perturbation}
\end{equation}
where $R_i^w=\sum_{j\in R(i)}r_{ij}$ and $U_i^w=\sum_{j\in U(i)}r_{ij}$ are the total reliability masses assigned to the two neighbor groups.

\noindent\textit{Control of the less reliable mass.}
Define $f(R,U)=U/(R+U+\epsilon)$. Its two partial derivatives have fixed signs:
\begin{equation}
\frac{\partial f}{\partial R}
=
-\frac{U}{(R+U+\epsilon)^2}
\le0,
\qquad
\frac{\partial f}{\partial U}
=
\frac{R+\epsilon}{(R+U+\epsilon)^2}
\ge0.
\label{eq:app_mass_monotonicity}
\end{equation}
The score separation in Theorem~\ref{thm:reliable_neighbor} gives $R_i^w\ge |R(i)|r_+$ and $U_i^w\le |U(i)|r_-$. Applying both inequalities according to the monotonicity in Eq.~\ref{eq:app_mass_monotonicity} yields
\begin{equation}
\frac{U_i^w}{R_i^w+U_i^w+\epsilon}
\le
\frac{|U(i)|r_-}
{|R(i)|r_++|U(i)|r_-+\epsilon}.
\label{eq:app_unreliable_mass_bound}
\end{equation}
This bound identifies the exact quantity controlled by reliability scoring: the normalized mass assigned to the perturbed neighbor group decreases as the separation between $r_+$ and $r_-$ increases.

\noindent\textit{Propagation through the completion gate.}
The center-node contribution in Eq.~\ref{eq:semantic_completion} is identical for both inputs and therefore cancels. Since $\mathbf{g}_i\in[0,1]^d$, element-wise multiplication by $1-\mathbf{g}_i$ is nonexpansive under the Euclidean norm. Consequently,
\begin{equation}
\begin{aligned}
\|\mathbf{s}_i-\mathbf{s}_i'\|_2
&=
\|(1-\mathbf{g}_i)\odot(\mathbf{m}_i-\mathbf{m}_i')\|_2\\
&\le
\|\mathbf{m}_i-\mathbf{m}_i'\|_2\\
&\le
\frac{|U(i)|r_-}
{|R(i)|r_++|U(i)|r_-+\epsilon}\,\delta.
\end{aligned}
\label{eq:app_completion_perturbation}
\end{equation}
The final line is exactly Eq.~\ref{eq:reliable_neighbor_bound}. \hfill$\square$

\noindent\textbf{Proof of Theorem~\ref{thm:shared_geometry}.}
The two inequalities in the theorem correspond to the two constraints in Eq.~\ref{eq:latent_constraints}. We establish them separately and then state their joint implication for the calibrated shared space.

\noindent\textit{Semantic recoverability.}
For each node, define the nonnegative reconstruction error
\begin{equation}
E_i
=
\|D_{\mathrm{sh}}(\mathbf{u}_i)-\mathbf{s}_i\|_2^2.
\label{eq:app_node_reconstruction_error}
\end{equation}
Stop-gradient changes the backward derivative while leaving the forward value intact. The first definition in Eq.~\ref{eq:latent_constraints} therefore implies
\begin{equation}
\mathbb{E}_{i\sim\operatorname{Unif}(\mathcal{V})}[E_i]
=
\frac{1}{N}\sum_{i=1}^{N}E_i
=
\mathcal{L}_{\mathrm{lrec}}.
\label{eq:app_mean_reconstruction_error}
\end{equation}
If $\mathcal{L}_{\mathrm{lrec}}=0$, nonnegativity forces every $E_i$ to be zero. Now consider $\mathcal{L}_{\mathrm{lrec}}>0$. For any $\kappa\in(0,1)$, Markov's inequality applied to Eq.~\ref{eq:app_mean_reconstruction_error} gives the complete probability chain
\begin{equation}
\begin{aligned}
\Pr_i\!\left[
\|D_{\mathrm{sh}}(\mathbf{u}_i)-\mathbf{s}_i\|_2
>
\sqrt{\mathcal{L}_{\mathrm{lrec}}/\kappa}
\right]
&=
\Pr_i\!\left[E_i>\mathcal{L}_{\mathrm{lrec}}/\kappa\right]\\
&\le
\frac{\mathbb{E}_i[E_i]}
{\mathcal{L}_{\mathrm{lrec}}/\kappa}\\
&=
\kappa.
\end{aligned}
\label{eq:app_recoverability_probability}
\end{equation}
This proves the node-level recoverability statement.

\noindent\textit{Reliable local geometry.}
For each directed reliable pair $(i,j)$, define its displacement error and edge-length distortion as
\begin{equation}
\mathbf{e}_{ij}
=
\Delta\mathbf{z}_{ij}^{\mathrm{anc}}-\Delta\mathbf{s}_{ij},
\qquad
d_{ij}
=
\left|
\|\Delta\mathbf{z}_{ij}^{\mathrm{anc}}\|_2
-
\|\Delta\mathbf{s}_{ij}\|_2
\right|.
\label{eq:app_geometry_errors}
\end{equation}
The reverse triangle inequality gives $d_{ij}\le\|\mathbf{e}_{ij}\|_2$. Since every reliability score is nonnegative, this pointwise relation can be squared, weighted, and summed without changing its direction:
\begin{equation}
\begin{aligned}
\sum_i\sum_{j\in R(i)}r_{ij}d_{ij}^2
&\le
\sum_i\sum_{j\in R(i)}r_{ij}\|\mathbf{e}_{ij}\|_2^2\\
&=
\sum_i\sum_{j\in R(i)}r_{ij}
\|\Delta\mathbf{z}_{ij}^{\mathrm{anc}}-\Delta\mathbf{s}_{ij}\|_2^2\\
&=
\mathcal{L}_{\mathrm{lcons}}.
\end{aligned}
\label{eq:app_weighted_geometry_error}
\end{equation}
The normalization constant $W_R=\sum_i\sum_{j\in R(i)}r_{ij}$ is strictly positive by assumption. Dividing Eq.~\ref{eq:app_weighted_geometry_error} by $W_R$ and taking square roots gives
\begin{equation}
\left(
\frac{1}{W_R}
\sum_i\sum_{j\in R(i)}r_{ij}
\left|
\|\Delta\mathbf{z}_{ij}^{\mathrm{anc}}\|_2
-
\|\Delta\mathbf{s}_{ij}\|_2
\right|^2
\right)^{1/2}
\le
\sqrt{\frac{\mathcal{L}_{\mathrm{lcons}}}{W_R}}.
\label{eq:app_geometry_rms_bound}
\end{equation}
This is the second inequality in Eq.~\ref{eq:shared_geometry_bound}. Together with Eq.~\ref{eq:app_recoverability_probability}, it shows that the two latent constraints control complementary properties: semantic decoding error at the node level and metric distortion over reliable local relations. \hfill$\square$

\subsection{Method Details}
\label{app:detailed_theoretical_analysis}

\noindent\textbf{Modality Absence Simulation.} For Independent Absence, we sample the simulated absence indicator for each observed node and modality pair as

\begin{equation}
s_i^{(m)}
\sim
\operatorname{Bernoulli}
\left(
\rho_m(e)
\right),
\quad
a_i^{(m)}=1,
\label{eq:app_independent_absence}
\end{equation}

where $\rho_m(e)$ denotes the absence probability of modality $m$ at epoch $e$. We then sample an additional reconstruction mask only from encoder-visible modalities,

\begin{equation}
r_i^{(m)}
\sim
\operatorname{Bernoulli}
\left(
\mu_m
\right),
\quad
a_i^{(m)}=1,
\quad
s_i^{(m)}=0.
\label{eq:app_reconstruction_mask}
\end{equation}

The encoder receives the effective availability in Eq.~\ref{eq:effective_availability}. The supervision set for modality reconstruction contains both sources of hidden observations,

\begin{equation}
\mathcal{T}^{(m)}
=
\left\{
i:
a_i^{(m)}=1,
\quad
s_i^{(m)}=1
\ \mathrm{or}\ 
r_i^{(m)}=1
\right\}.
\label{eq:app_reconstruction_targets}
\end{equation}

Thus, simulated absence evaluates recovery from the remaining modalities and topology, while the reconstruction mask supplies additional self-supervised targets without exposing their features to the encoder.

For Neighborhood Correlated Absence, we first sample seed nodes $\mathcal{C}^{(m)}$ and mask the neighborhoods within $K$ hops of the seeds.

\begin{equation}
s_i^{(m)}
=
\mathbb{I}
\left[
\min_{v_c\in\mathcal{C}^{(m)}}
\operatorname{dist}_{\mathcal{G}}(v_i,v_c)
\le
K
\right].
\label{eq:app_neighborhood_absence}
\end{equation}

The number of seeds is selected to match the target missing ratio. For Property Dependent Absence, a nonnegative node property $g_i$ determines the absence probability.

\begin{equation}
\Pr
\left(
s_i^{(m)}=1
\mid
g_i
\right)
=
\min
\left(
1,
c_m(e)\psi_m(g_i)
\right),
\label{eq:app_conditional_absence}
\end{equation}

where $c_m(e)$ is chosen to attain the desired average missing ratio and $\psi_m$ is a nonnegative scaling function. At inference, we remove the simulated indicators and use the observed availability $a_i^{(m)}$ directly.

\noindent\textbf{Reliable Neighborhood Semantic Completion.}
The topology input $\mathbf{p}_i$ combines Laplacian positional features~\cite{dwivedi2023benchmarking} and random walk~\cite{dwivedi2022lspe} return probabilities computed on the symmetric binary graph. The former use the eigenvectors associated with the smallest normalized Laplacian eigenvalues, excluding the first eigenvector, while the latter use diagonal entries of successive powers of the transition matrix. The encoder $f_{\mathrm{topo}}$ maps these features through two linear layers with LayerNorm and GELU between them.

In Eq.~\ref{eq:reliability_score}, $\Phi$ concatenates the two semantic states, the two topology encodings, their respective element-wise products, and $\omega_{ij}$, then applies a linear projection and GELU. The final affine scoring layer includes the bias $b_r$. For the row-normalized transition matrix $\mathbf{T}$, the diffusion matrix is $\mathbf{W}=\eta\sum_{k=0}^{L}(1-\eta)^k\mathbf{T}^k$, where $\eta$ is the restart coefficient and $L$ is the diffusion depth. The weight for $j\rightarrow i$ is $\omega_{ij}=W_{ji}/W_{jj}$, bounded below by a small positive constant. This normalizes the source-to-target connection by the source's self-diffusion. Protocols that omit diffusion use unit edge weights. In Eq.~\ref{eq:semantic_completion}, all incoming neighbors contribute through normalized reliability scores. The retention gate is a sigmoid applied to an affine projection of $[\mathbf{h}_i^{(0)};\mathbf{t}_i;\bar r_i]$, where $\bar r_i$ is their mean reliability, yielding a separate balance for each feature dimension.

\noindent\textbf{Topology-Guided Shared Latent Calibration.}
In Eq.~\ref{eq:topology_calibration}, $P_{\mathrm{sh}}$ combines a two-layer projection with a gated linear semantic residual. The projection uses LayerNorm and GELU between its linear layers, and a sigmoid gate computed from $\mathbf{s}_i$ controls the residual. The topology conditioner uses two linear layers separated by GELU to predict scale and shift offsets. The scale is parameterized as $\boldsymbol{\gamma}_i=\mathbf{1}+\Delta\boldsymbol{\gamma}_i$, and the output layer starts at zero, so the affine calibration initially preserves $\mathbf{u}_i$. The bypass gate $\mathbf{b}_i$ is a sigmoid applied to an affine projection of $[\mathbf{s}_i;\mathbf{t}_i]$, retaining completed semantics directly in the final representation.

For Eq.~\ref{eq:latent_constraints}, $D_{\mathrm{sh}}$ is a two-layer decoder with GELU that reconstructs $\mathbf{s}_i$ from the shared state before topology calibration. Both constraints detach their completed-semantic targets. To construct $R(i)$, we rank incoming neighbors by reliability, retain a bounded number of the highest-scoring neighbors, and apply a reliability threshold. If none passes, the highest-scoring neighbor is retained with its continuous reliability weight. The implementation averages reconstruction errors over nodes and coordinates, and consistency errors over selected edges and coordinates. These are normalized versions of the losses in Eq.~\ref{eq:latent_constraints}, whose unnormalized notation is used in the theoretical bounds.

\noindent\textbf{Pretraining Objectives.}

\noindent\textit{Contrastive Alignment.}
GraphBind uses contrastive learning~\cite{radford2021clip,girdhar2023imagebind} to bind each shared representation to node-local modality semantics and structurally related semantics. Multiple positives provide complementary semantic targets when a node's own modalities are incomplete. For a pretraining batch $\mathcal{B}$, let $\mathcal{V}_m=\{i\in\mathcal{B}:a_i^{(m)}=1\}$ collect nodes whose original modality $m$ is available. We compute the cosine similarity $q_{ij}^{(m)}$ between anchor $i$ and the modality representation of node $j$ from the original, unmasked input stream.

Besides the same node, the positive set $\mathcal{P}_{i,+}^{(m)}$ contains candidates with sufficiently high diffusion weight and cross-modal similarity, ranked by the product of their diffusion weight and nonnegative similarity. Here, $d_{ij}=W_{ij}/W_{ii}$ measures diffusion from anchor $i$ to candidate $j$. Hard negatives $\mathcal{H}_i^{(m)}$ have high similarity and low diffusion weight, and are ranked by $q_{ij}^{(m)}-d_{ij}$. Each selection retains a bounded number of candidates. These hard negatives emphasize semantic distinctions that graph connections help resolve. Feature dropout additionally provides an augmented modality view, whose dot product with the normalized anchor is denoted by $\widetilde q_{ij}^{(m)}$. The objective is
\begin{equation}
\begin{aligned}
A_i^{(m)}
&=e^{q_{ii}^{(m)}/\tau}+e^{\widetilde q_{ii}^{(m)}/\tau}
+\sum_{j\in\mathcal{P}_{i,+}^{(m)}}e^{q_{ij}^{(m)}/\tau},\\
B_i^{(m)}
&=\sum_{j\in\mathcal{V}_m}d_{ij}e^{q_{ij}^{(m)}/\tau}
+\sum_{j\in\mathcal{V}_m}e^{\widetilde q_{ij}^{(m)}/\tau}
+\lambda_{\mathrm{hard}}\sum_{j\in\mathcal{H}_i^{(m)}}e^{q_{ij}^{(m)}/\tau},\\
\mathcal{L}_{\mathrm{ctr}}^{(m)}
&=-\frac{1}{|\mathcal{V}_m|}\sum_{i\in\mathcal{V}_m}
\log\frac{A_i^{(m)}}{B_i^{(m)}}.
\end{aligned}
\label{eq:contrastive_loss}
\end{equation}
Here, $\tau>0$ is the temperature and $\lambda_{\mathrm{hard}}$ controls hard-negative emphasis. The denominator is evaluated over all available candidates in chunks. Disabled augmentation or mining removes the corresponding terms, and unavailable diffusion weights are replaced by one. Summing over modalities gives $\mathcal{L}_{\mathrm{ctr}}=\sum_m\mathcal{L}_{\mathrm{ctr}}^{(m)}$. Aggregating positive scores allows their contributions to follow semantic agreement, while hard-negative emphasis encourages finer distinctions within the shared space.

\noindent\textit{Masked Modality Reconstruction.}
Contrastive alignment primarily preserves relations among samples and can omit modality details needed by generative tasks. Following masked graph autoencoding~\cite{hou2022graphmae}, GraphBind reconstructs the hidden observations indexed by $\mathcal{T}^{(m)}$. We use the scaled cosine error
\begin{equation}
\delta_{\mathrm{sce}}(\mathbf{u},\mathbf{v})
=\left(1-
\frac{\mathbf{u}^{\top}\mathbf{v}}
{\max(\|\mathbf{u}\|_2,\epsilon)\max(\|\mathbf{v}\|_2,\epsilon)}
\right)^{\gamma},\qquad \gamma\ge1.
\label{eq:app_scaled_cosine_error}
\end{equation}
Cosine reconstruction matches the angular geometry of the modality features and removes feature magnitude from the reconstruction error. This focuses supervision on semantic direction, which MSE also couples to feature norms. The implementation uses $\gamma=1$. The masked reconstruction loss is
\begin{equation}
\mathcal{L}_{\mathrm{rec}}^{(m)}
=\frac{1}{|\mathcal{T}^{(m)}|}
\sum_{i\in\mathcal{T}^{(m)}}
\delta_{\mathrm{sce}}\left(D_m(\mathbf{z}_i^{\mathrm{anc}}),\mathbf{x}_i^{(m)}\right),
\label{eq:reconstruction_loss}
\end{equation}
where $D_m$ is the modality decoder and $\mathcal{L}_{\mathrm{rec}}=\sum_m\mathcal{L}_{\mathrm{rec}}^{(m)}$. Reconstructing hidden attributes encourages the shared representation to retain modality information recovered through the semantic completion path.

\noindent\textit{Latent Decorrelation.}
Alignment and reconstruction can encode correlated information repeatedly across latent dimensions, reducing the effective capacity of the shared space. Let $\bar{\mathbf{Z}}$ and $\bar{\mathbf{V}}^{(m)}$ denote column-wise centered and standardized anchor and modality matrices. Their empirical correlation matrices are
\begin{equation}
\mathbf{C}^{(a)}=\frac{1}{|\mathcal{B}|}\bar{\mathbf{Z}}^{\top}\bar{\mathbf{Z}},
\qquad
\mathbf{C}^{(m)}=\frac{1}{|\mathcal{V}_m|}(\bar{\mathbf{V}}^{(m)})^{\top}\bar{\mathbf{V}}^{(m)}.
\label{eq:app_correlation_matrices}
\end{equation}
Following graph decorrelation learning~\cite{zhang2021ccassg}, we penalize their off-diagonal entries:
\begin{equation}
\mathcal{L}_{\mathrm{decor}}
=\frac{1}{d(d-1)}\left[
\sum_{u\ne v}(\mathbf{C}_{uv}^{(a)})^2
+\sum_{m\in\mathcal{M}}\sum_{u\ne v}(\mathbf{C}_{uv}^{(m)})^2\right],
\label{eq:decorrelation_loss}
\end{equation}
where $d$ is the representation dimension. Penalizing correlation across distinct coordinates encourages complementary features for multimodal binding.

\noindent\textit{Total Objective.}
We jointly optimize semantic alignment, modality recovery, and shared-space regularization through
\begin{equation}
\begin{aligned}
\mathcal{L}_{\mathrm{pre}}
&=\lambda_{\mathrm{ctr}}\mathcal{L}_{\mathrm{ctr}}
+\lambda_{\mathrm{rec}}\mathcal{L}_{\mathrm{rec}}
+\lambda_{\mathrm{dec}}\mathcal{L}_{\mathrm{decor}}\\
&\quad+\lambda_{\mathrm{lr}}\mathcal{L}_{\mathrm{lrec}}
+\lambda_{\mathrm{lc}}\mathcal{L}_{\mathrm{lcons}},
\end{aligned}
\label{eq:pretraining_objective}
\end{equation}
where the five coefficients balance the objectives. Together, these terms train the shared representation to support cross-modal alignment and missing-modality completion while preserving useful semantic variation.

\noindent\textbf{Topology Prefix Prompter.}
In Eq.~\ref{eq:prompt_attention}, $\operatorname{Attn}$ is a scaled dot product between a projected query and projected prompt keys. The query combines the shared node representation with a linearly projected task token, and the values are the original prompt vectors. The selected context, shared representation, and projected task token are added and normalized before a two-layer MLP with GELU produces the residual update. A feature-wise task gate, obtained by applying $\tanh$ to an affine projection of the projected task token, scales this update before it is added to $\mathbf{z}_i^{\mathrm{anc}}$. The residual output layer is initialized to zero, giving an identity mapping at initialization. When no task token is supplied, the query uses a zero condition and the residual update is applied directly.

\noindent\textbf{G2M-Adapter.}
For Eq.~\ref{eq:topology_condition_sequence}, the context sequence places the target node first, followed by its incoming one-hop neighbors from the task metadata. Repeated entries and the target itself are removed from the neighbor list, which is truncated in metadata order. Unused slots contain zero vectors and are masked as attention keys.

For text generation, $F_s$ projects the shared representations and applies local message passing with bidirectional links from the target to its neighbors and between consecutive retained neighbors. Each block combines transformed self states and mean neighbor states through GELU, a residual connection, and LayerNorm. Slot embeddings distinguish context positions. Learnable queries aggregate these states through the gated attention in Eq.~\ref{eq:topology_condition_sequence}, and $O_s$ applies LayerNorm. The resulting tokens are prepended to the language input embeddings, with prefix length determined by the query bank.

For image generation, $F_s$ is a linear projection and $\mathbf{Q}_s=\mathbf{E}_s$ gives self-attention over context nodes. After the gated residual update, $O_s$ adds slot embeddings, normalizes and projects the states, and zeros invalid slots. A final linear projection expands each slot into generator tokens followed by LayerNorm, retaining the padded slots in the output sequence. These tokens are concatenated with text embeddings and supplied to the diffusion U-Net through its existing cross-attention interface.

\subsection{Pseudocode of \textsc{GraphBind}}
\label{app:pseudocode}

We provide the overall procedure for GraphBind in Algorithm~\ref{alg:graphbind}.

\begin{algorithm}[htbp]\fontsize{8pt}{4pt}\selectfont
\DontPrintSemicolon
\SetAlgoLined
\caption{Overall Procedure of GraphBind}
\label{alg:graphbind}

\KwInput{Pretraining graphs $\mathbb{G}_{\mathrm{pre}}$ with modality availability; maximum epochs $E$; downstream graph and training data for task $s$.}
\KwOutput{Pretrained encoder $f_{\Theta}$ and adapted task model.}

\tcc{Initialization}
initialize model parameters and extract features with frozen modality encoders; \\
prepare topology features $\mathbf{p}_i$ and diffusion weights; \\

\For{$e=1$ to $E$}{
    \ForEach{graph batch from $\mathbb{G}_{\mathrm{pre}}$}{
        \tcc{Step 1: Modality Absence Simulation}
        sample absence and reconstruction masks, then set
        $\tilde a_i^{(m)} \leftarrow a_i^{(m)}(1-s_i^{(m)})(1-r_i^{(m)})$; \\

        \tcc{Step 2: Self-Semantic Initialization and Neighborhood Completion}
        initialize $\mathbf{h}_i^{(0)}$ from visible modalities via Eq.~(\ref{eq:self_initialization}); \\
        compute $\mathbf{t}_i \leftarrow f_{\mathrm{topo}}(\mathbf{p}_i)$ and neighbor reliability $r_{ij}$ via Eq.~(\ref{eq:reliability_score}); \\
        aggregate incoming neighbors and obtain $\mathbf{s}_i$ via Eq.~(\ref{eq:semantic_completion}); \\

        \tcc{Step 3: Topology-Guided Shared Latent Calibration}
        project $\mathbf{u}_i \leftarrow P_{\mathrm{sh}}(\mathbf{s}_i)$ and calibrate $\mathbf{z}_i^{\mathrm{anc}}$ via Eq.~(\ref{eq:topology_calibration}); \\
        select reliable neighbors $R(i)$ and compute latent constraints via Eq.~(\ref{eq:latent_constraints}); \\
        optimize $\mathcal{L}_{\mathrm{pre}}$ with original modality targets via Eq.~(\ref{eq:pretraining_objective}); \\
    }
}

\tcc{Step 4: Unified Downstream Adaptation}
initialize the task model from $f_{\Theta}$ and its task interface; \\
\For{each task training batch}{
    encode task-visible modalities and graph structure to obtain $\mathbf{z}_i^{\mathrm{anc}}$; \\
    \eIf{$s$ is node classification or link prediction}{
        select prompt context via Eq.~(\ref{eq:prompt_attention}) and apply the residual prompt update; \\
        predict with the downstream head; \\
    }{
        gather the target node and one-hop neighbors into $\mathbf{Z}_{\mathrm{topo}}$; \\
        map $\mathbf{Z}_{\mathrm{topo}} \rightarrow \bar{\mathbf{Z}}_s$ via Eq.~(\ref{eq:topology_condition_sequence}); \\
        supply $\bar{\mathbf{Z}}_s$ as language prefixes or diffusion context tokens; \\
    }
    optimize the task objective following Appendix~\ref{app:experiment_setup}; \\
}

\Return{pretrained encoder $f_{\Theta}$ and adapted task model};
\end{algorithm}

\subsection{Dataset Details}
\label{app:dataset_details}

In this section, we provide detailed statics of the datasets in Table~\ref{tab:dataset_statistics}. Texual descriptions are shown as follows.

\begingroup
\setlength{\intextsep}{4pt}
\captionsetup{skip=3pt}
\begin{table}[H]
\centering
\caption{Statistics of the datasets used in GraphBind.}
\label{tab:dataset_statistics}
\fontsize{9pt}{10.5pt}\selectfont
\renewcommand{\arraystretch}{1.0}
\setlength{\tabcolsep}{10pt}
\begin{tabular*}{\textwidth}{@{\extracolsep{\fill}}cccccc}
\toprule
\textbf{Dataset} & \textbf{Modalities} & \textbf{Nodes} & \textbf{Edges} & \textbf{Classes} & \textbf{Domain} \\
\midrule
Grocery       & Text, Visual & 17,074 & 171,340   & 20 & E-Commerce \\
RedditS       & Text, Visual & 15,894 & 566,160   & 20 & Social Media \\
Movies        & Text, Visual & 16,672 & 218,390   & 20 & E-Commerce \\
Toys          & Text, Visual & 20,695 & 126,886   & 18 & E-Commerce \\
Ele-fashion   & Text, Visual & 97,766 & 199,602   & 12 & E-Commerce \\
Sports        & Text, Visual & 50,250 & 356,202   & -- & E-Commerce \\
Bili\_Dance  & Text, Visual & 2,307  & 9,127     & -- & Video Recommendation \\
DY            & Text, Visual & 8,299  & 35,627    & -- & Video Recommendation \\
SemArt        & Text, Visual & 21,382 & 1,216,432 & -- & Art Network \\
Flickr30k     & Text, Visual & 31,783 & 181,151   & -- & Image Network \\
\bottomrule
\end{tabular*}
\end{table}
\endgroup

\noindent\textbf{Grocery, Movies, and Toys.} These datasets are extracted from the Amazon product review corpus~\cite{Amazon2018}. Nodes represent products, and edges encode purchase relations derived from the also bought metadata. Textual attributes include product titles, descriptions, and reviews, while visual attributes are product images. Grocery and Toys are used for node classification, and Movies is used for few shot node classification. Labels are given by product categories or genres.

\noindent\textbf{RedditS.} RedditS is derived from the RedCaps collection of Reddit posts with images and text~\cite{RedditS}. We follow the OpenMAG preprocessing protocol to construct its post graph~\cite{wan2026openmag}. Post titles and bodies provide textual attributes, and embedded images provide visual attributes. We use RedditS for few shot node classification, where labels identify subreddit communities.

\noindent\textbf{DY and Bili\_Dance.} These datasets are drawn from the NineRec benchmark~\cite{ninerec}. Nodes represent short videos from Douyin and Bilibili, respectively, and edges reflect user co-viewing or co-interaction behavior. Textual attributes are obtained from titles, descriptions, or hashtags, and visual attributes from video covers or sampled frames. Both datasets are used for link prediction.

\noindent\textbf{Ele-fashion.} Ele-fashion is a heterogeneous Amazon graph that merges electronics and fashion products~\cite{Amazon2018,hou2024bridging}. Nodes represent products from the two categories, and edges encode cross-category co-purchases. Textual attributes combine technical specifications with style descriptions, while visual attributes are product images. We use Ele-fashion for cross-domain pretraining.

\noindent\textbf{Sports.} Sports is drawn from Amazon's Sports and Outdoors category~\cite{Amazon2018,hou2024bridging}. Nodes represent sports equipment and accessories, and edges capture complementary or co-purchase relations. Textual attributes comprise product specifications, and visual attributes are product images. We use Sports for cross-domain pretraining.

\noindent\textbf{Flickr30k.} Flickr30k Entities provides paired images and sentences with grounding annotations between regions and phrases~\cite{flickr30k}. Following the OpenMAG construction~\cite{wan2026openmag}, we link image regions and caption phrases through their grounding relations and evaluate G2Text generation.

\noindent\textbf{SemArt.} SemArt is a fine-art collection introduced for multimodal retrieval and semantic art understanding~\cite{semart}. Following the OpenMAG construction~\cite{wan2026openmag}, nodes represent paintings and edges connect artworks sharing metadata such as artist, period, or school. Textual attributes comprise titles and expert descriptions, while visual attributes are the corresponding painting images. We use SemArt for G2Image generation.

\subsection{Compared Baselines}
\label{app:compared_baselines}

We compare GraphBind with representative methods from MAG models, GFMs, and MGFMs. Mario~\cite{sun2026mario} focuses on adapting LLMs through modality-adaptive graph instruction tuning, with downstream evaluation centered on node classification and link prediction. We leave Mario outside the experimental comparison given its emphasis on LLM-based graph reasoning, while our study mainly focuses on learning shared graph representations and adapting them to both discriminative and generative tasks.

\noindent\textbf{MMGCN}~\cite{mmgcn} is a multimodal graph convolution model for micro video recommendation. It builds a separate user and item graph for each modality, propagates information independently within each graph, and combines the resulting user and item representations with intrinsic features and shared ID embeddings for personalized ranking.

\noindent\textbf{MGAT}~\cite{mgat} is a multimodal graph attention model for recommendation. It applies gated attention on parallel modality-specific interaction graphs, learns adaptive weights for neighboring items and modalities, and fuses the resulting signals to capture fine-grained user preferences under noisy or conflicting modalities.

\noindent\textbf{LGMRec}~\cite{lgmrec} combines local graph learning with a global hypergraph module to model collaborative and multimodal user interests for recommendation. Its local branch separates collaborative and modality-specific signals, while the hypergraph branch captures higher-order dependencies among users and items before the two representations are integrated.

\noindent\textbf{TMTE}~\cite{zhu2026tmte} jointly and iteratively evolves graph topology and multimodal representations under downstream task supervision, forming a task-aware co-evolution framework for multimodal graph learning. It learns multi-perspective similarity relations to refine the adjacency structure and uses the evolved topology for smoothness-regularized multimodal fusion and cross-modal alignment.

\noindent\textbf{RoleMAG}~\cite{zuo2026rolemag} routes neighbors according to learned roles and propagates multimodal information through role-specific experts, targeting robust representation learning on multimodal graphs. Its router decomposes each edge into shared, complementary, and heterophilous relations, and the corresponding experts perform smoothing, cross-modal completion, and signed structural filtering before gated fusion.

\noindent\textbf{GIT}~\cite{wang2025git} represents graph computation patterns as task trees and pretrains a shared graph model through task-tree reconstruction. A task-tree representation unifies node-, edge-, and graph-level objectives, after which the pretrained encoder supports supervised, few-shot, zero-shot, and in-context adaptation across graph domains.

\noindent\textbf{GFT}~\cite{wang2024gft} learns a transferable vocabulary of computation trees with vector quantization and uses it to unify graph pretraining across domains and downstream fine tuning. During pretraining, graph computation trees are reconstructed into a shared codebook. During adaptation, the resulting tree tokens provide a common interface for node, link, and graph prediction.

\noindent\textbf{OpenGraph}~\cite{xia2024opengraph} distills graph generalization from large language models through a unified graph tokenizer, a scalable graph Transformer, and LLM-assisted graph data augmentation. The tokenizer converts heterogeneous graph neighborhoods into transferable tokens, while the Transformer models global node dependencies for zero-shot transfer across unseen graph domains and tasks.

\noindent\textbf{RiemannGFM}~\cite{sun2025riemanngfm} formulates graph representation learning in a Riemannian geometric space and pretrains a transferable graph foundation model for cross-domain node and link prediction. Its geometric encoder and task-agnostic pretraining objective are designed to accommodate graphs with diverse structural properties while retaining transferable relational geometry.

\noindent\textbf{UniGraph2}~\cite{he2025unigraph2} aligns text and image features with graph structure through frozen modality encoders, a mixture-of-experts alignment module, and a universal GNN. The experts select modality-specific alignment functions before topology-aware message passing, producing a unified embedding space that transfers to node classification and link prediction.

\noindent\textbf{PLANET}~\cite{liu2026planet} is a multimodal graph foundation model based on a divide-and-conquer design that separately models graph structure and modality semantics before integrating them for transfer across discriminative and generative tasks. Its modular decomposition isolates graph encoding, multimodal representation learning, and task adaptation, enabling a shared foundation model to support heterogeneous downstream objectives.

\subsection{Experiment Setup}
\label{app:experiment_setup}

\noindent\textbf{Experimental Environment.} All experiments are implemented in Python 3.10.18 with PyTorch 2.8 and CUDA 12.9. We use a workstation equipped with Intel Xeon Scalable processors, 256~GB system memory, and NVIDIA RTX 6000 Ada GPUs with 48~GB VRAM each.

\noindent\textbf{Pretraining Data and Sampling.} We pretrain on Toys, Grocery, Movies, RedditS, Ele-fashion, and Sports. Each dataset contributes one graph update per epoch, giving every dataset one sixth of the updates. Ele-fashion and Sports use sampled subgraphs to control memory use. Textual and visual node attributes are encoded as 768 dimensional vectors by CLIP ViT L/14~\cite{radford2021clip} and used as the two modality inputs. DY, Bili\_Dance, Flickr30k, and SemArt are reserved for adaptation and evaluation. GraphBind uses a 768 dimensional shared space, two topology GNN layers, a 32 dimensional topology input, eight attention heads, and an eight token topology prompt. Independent absence uses a 10\% missing probability for each modality. Neighborhood correlated absence masks one hop neighborhoods of seed nodes with a 3\% seed ratio. Property dependent absence scales the missing probability by node properties and sets the average missing ratio to 10\%~\cite{wu2026missingmodality}. Modalities visible to the encoder are additionally masked at 15\% to provide reconstruction targets.

\noindent\textbf{Pretraining Optimization.} We optimize all GraphBind parameters with AdamW~\cite{loshchilov2019decoupled} for at most 500 epochs using learning rate $1\times10^{-4}$, weight decay $1\times10^{-5}$, gradient clipping at norm 1.0, and cosine decay to $1\times10^{-6}$ with 20 warmup epochs. Early stopping is enabled after epoch 300 with patience 150 and minimum improvement $10^{-4}$. The InfoNCE temperature is 0.07 and its chunk size is 512. Reliable neighbor completion uses two layers, top $4$ latent consistency neighbors, and a reliability threshold of 0.2. The pretraining objective combines contrastive and generative losses with unit weights, decorrelation weight 0.05, latent reconstruction weight 0.3, and latent consistency weight 0.2. The cross modal reconstruction coefficient is 0.5, and hard negative mining considers at most 4,096 nodes with top $4$ hard negatives.

\noindent\textbf{Discriminative Fine Tuning.} For node classification, we use a 0.6/0.2/0.2 train/validation/test split generated by a seeded random permutation. We train for 100 epochs with AdamW~\cite{loshchilov2019decoupled} using learning rate $10^{-3}$ and weight decay $10^{-5}$, and stop after 30 epochs without validation improvement. The topology prompt has dropout 0.5. The text classifier is selected with fixed fusion weight 1.0. Few shot node classification uses $K\in\{3,5\}$ labeled nodes per class, 10 episodes, and the same prompt interface. For link prediction, we train for 200 epochs with learning rate $10^{-3}$, weight decay $10^{-5}$, and patience 50. We use 0.1/0.2 validation test ratios, one training negative and 1,000 evaluation negatives per positive edge, edge and score batches of 2,048, two hop neighbor fanouts of 15 and 15, and a three layer MLP predictor with hidden size 256 and dropout 0.02. The encoder learning rate scale is 0.1, and the gradient clipping threshold is 1.0. Our implementation computes Hits@1, Hits@3, and Hits@10; we report MRR and Hits@3.

\noindent\textbf{Generative Fine Tuning.} For G2Text, we follow the OpenMAG task formulation~\cite{wan2026openmag}. All methods use OPT-125M~\cite{zhang2022opt} with LoRA~\cite{hu2022lora} of rank 64 and scaling 128 applied to \texttt{q\_proj} and \texttt{v\_proj}. GraphBind constructs graph prefixes with the G2M-Adapter, while the compared baselines use a shared prefix projection module. We train for 15 epochs with batch size 4 per device, learning rate $10^{-4}$, weight decay $10^{-2}$, AdamW~\cite{loshchilov2019decoupled} parameters $(\beta_1,\beta_2)=(0.9,0.95)$, gradient clipping 1.0, maximum input and output lengths 96 and 64, and beam search with four beams with maximum generation length 64. The graph backend has hidden size 768, two layers, eight heads, dropout 0.1, and at most eight neighbor slots. For G2Image, we follow the OpenMAG G2Image protocol~\cite{wan2026openmag} and provide graph tokens to Stable Diffusion v1.5~\cite{rombach2022high}. We train at resolution $256\times256$ for 10 epochs with batch size 16, learning rate $5\times10^{-5}$, weight decay $10^{-2}$, gradient clipping 1.0, and graph context dropout 0.1. Inference uses 50 diffusion steps, guidance scale 7.5, graph context, and up to 4,277 evaluation samples. All methods finetune the Stable Diffusion v1.5 U-Net while keeping the VAE and text encoder frozen. Generated images are evaluated with CLIP ViT-L/14 and DINOv2-large~\cite{oquab2023dinov2}.

\noindent\textbf{Baseline Configuration.} The eleven baselines comprise five MAG models given by MMGCN, MGAT, LGMRec, TMTE, and RoleMAG, four GFMs given by GIT, GFT, OpenGraph, and RiemannGFM, and two MGFMs given by UniGraph2 and PLANET. For the MAG models, hidden dimensions are 128 for MMGCN, MGAT, and LGMRec and 256 for RoleMAG, with their standard layer counts and dropout values. TMTE uses hidden size 64, ten evolution iterations, four similarity perspectives, graph skip coefficient 0.8, and topology update ratio 0.1. RoleMAG uses three shared and backbone layers, modality weights $\lambda_v=\lambda_t=0.5$, and top 16 complementary neighbors. GIT and GFT use two layer GraphSAGE backbones with hidden size 256. GFT additionally uses a 256 entry codebook. OpenGraph uses a two layer GNN and a two layer graph Transformer with 256 dimensional tokens and 64 anchors. RiemannGFM uses 128 dimensional hidden states and 64 dimensional embeddings. UniGraph2 uses a 768 dimensional three layer encoder with eight experts and top 2 routing. PLANET uses 768 dimensional four layer Transformer blocks, five experts with top 2 routing, a 20,480 entry codebook, and 10 dimensional Laplacian positional encodings. Pretraining checkpoints and task specific optimization settings are selected consistently across methods.

\noindent\textbf{Node Classification.} We predict the class of each node on Grocery and Toys. The graph is split into 60\%/20\%/20\% train, validation, and test nodes. A representation from the topology prompt is fed to a linear text classifier trained with cross entropy for 100 epochs using learning rate $10^{-3}$, weight decay $10^{-5}$, and patience 30. Accuracy is computed as
\begin{equation}
\mathrm{Acc}=\frac{1}{N}\sum_{i=1}^{N}\mathbb{I}(\hat y_i=y_i),
\end{equation}
and macro-F1 as
\begin{equation}
\mathrm{F1}=\frac{1}{C}\sum_{c=1}^{C}\frac{2P_cR_c}{P_c+R_c},\qquad
P_c=\frac{TP_c}{TP_c+FP_c},\quad R_c=\frac{TP_c}{TP_c+FN_c}.
\end{equation}

\noindent\textbf{Few Shot Node Classification.} We evaluate Movies and RedditS with an episodic $N$ way $K$ shot protocol using $N=5$ and $K\in\{3,5\}$. In each of 10 episodes, five eligible classes are sampled. Exactly $K$ support nodes are sampled without replacement from the training partition for every selected class, while all validation and test nodes belonging to the selected classes form the corresponding query sets. Labels are remapped to the class index set of the episode. The classifier is selected by validation macro-F1, and accuracy and macro-F1 are averaged across episodes.

\noindent\textbf{Link Prediction.} On DY and Bili\_Dance, we learn to rank missing destination nodes for each source using one negative edge during training and 1,000 sampled negatives at evaluation. The model is trained for 200 epochs with learning rate $10^{-3}$, weight decay $10^{-5}$, a three-layer 256-hidden-unit predictor (dropout 0.02), and [15,15] two-hop neighbor sampling. For query $q$, the positive rank is
\begin{equation}
r_q=\sum_j\mathbb{I}(s_{qj}\ge s_q^+).
\end{equation}
We report MRR and Hits@K as
\begin{equation}
\mathrm{MRR}=\frac{1}{|Q|}\sum_{q\in Q}\frac{1}{r_q},
\end{equation}
\begin{equation}
\mathrm{Hits@K}=\frac{1}{|Q|}\sum_{q\in Q}\mathbb{I}(r_q\le K),\qquad K\in\{1,3,10\}.
\end{equation}

\noindent\textbf{G2Text.} Following the OpenMAG Flickr30k graph-to-text protocol~\cite{wan2026openmag}, the model generates a caption for each center node from its sampled multimodal neighbors. We use the OPT-125M/LoRA setup above. The main experiments report BLEU-4~\cite{papineni2002bleu} and ROUGE-L~\cite{lin2004rouge}, while CIDEr~\cite{vedantam2015cider} is additionally provided as a supplementary metric. For candidate corpus $c$ and reference corpus $r$, BLEU-4 is
\begin{equation}
\mathrm{BLEU\mbox{-}4}=\mathrm{BP}\exp\left(\sum_{n=1}^{4}w_n\log p_n\right),\qquad w_n=\frac{1}{4},
\end{equation}
where $p_n$ is the modified $n$-gram precision and
\begin{equation}
\mathrm{BP}=\begin{cases}1,&|c|>|r|,\\ \exp(1-|r|/|c|),&|c|\le |r|.\end{cases}
\end{equation}
ROUGE-L uses the longest common subsequence length $L_{\mathrm{LCS}}$ given below.
\begin{equation}
R_{\mathrm{LCS}}=\frac{L_{\mathrm{LCS}}}{|r|},\qquad P_{\mathrm{LCS}}=\frac{L_{\mathrm{LCS}}}{|c|},
\end{equation}
\begin{equation}
\mathrm{ROUGE\mbox{-}L}=\frac{(1+\beta^2)R_{\mathrm{LCS}}P_{\mathrm{LCS}}}{R_{\mathrm{LCS}}+\beta^2P_{\mathrm{LCS}}}.
\end{equation}
CIDEr averages the cosine similarity between TF IDF weighted $n$ gram vectors over $n=1,\ldots,4$ and $M$ references.
\begin{equation}
\mathrm{CIDEr}(c,r)=\frac{10}{4M}\sum_{n=1}^{4}\sum_{i=1}^{M}\frac{g_n(c)\cdot g_n(r_i)}{\lVert g_n(c)\rVert_2\lVert g_n(r_i)\rVert_2}.
\end{equation}

\noindent\textbf{G2Image.} Following the OpenMAG SemArt graph to image protocol~\cite{wan2026openmag}, diffusion guided by graph tokens synthesizes the image associated with each center node. We report CLIP-S and DINO-S. Let $e^{\mathrm{CLIP}}_{\mathrm{gen},i}$ and $e^{\mathrm{CLIP}}_{\mathrm{ref},i}$ denote the $\ell_2$ normalized CLIP ViT-L/14 image features~\cite{radford2021clip}. CLIP S is computed as
\begin{equation}
\mathrm{CLIP\mbox{-}S}=\frac{1}{N}\sum_{i=1}^{N}\max\left(\left\langle e^{\mathrm{CLIP}}_{\mathrm{gen},i},e^{\mathrm{CLIP}}_{\mathrm{ref},i}\right\rangle,0\right).
\end{equation}
Similarly, with $\ell_2$-normalized DINOv2 features~\cite{oquab2023dinov2}, DINO-S is
\begin{equation}
\mathrm{DINO\mbox{-}S}=\frac{1}{N}\sum_{i=1}^{N}\max\left(\left\langle e^{\mathrm{DINO}}_{\mathrm{gen},i},e^{\mathrm{DINO}}_{\mathrm{ref},i}\right\rangle,0\right).
\end{equation}
Percentage-style tables multiply these means by 100.

\end{document}